\documentclass[nohyperref]{article}
\usepackage[preprint]{neurips_2023}

\usepackage{microtype}
\usepackage{graphicx}
\usepackage{subfigure}
\usepackage{booktabs} 
\usepackage[colorlinks,citecolor=blue,urlcolor=blue,bookmarks=false,hypertexnames=true,linkcolor=blue]{hyperref}
\usepackage{alias}
\usepackage{amsmath,amsfonts,amssymb,amsthm}
\usepackage{xcolor}
\usepackage{mdframed}

\usepackage{thm-restate}
\usepackage{graphicx}
\usepackage{url}
\usepackage{booktabs}
\usepackage{tabu}

\usepackage{amsmath}
\usepackage{amssymb}
\usepackage{mathtools}
\usepackage{caption}
\usepackage[capitalize,noabbrev]{cleveref}

\title{Robust Linear Regression: Phase-Transitions and Precise Tradeoffs for General Norms}
\date{Early March, 2023}
\usepackage{times}

\usepackage{afterpage}

\usepackage[toc,page,header]{appendix}
\usepackage{minitoc}

\author{
Elvis Dohmatob\\
Meta -- FAIR\\
\texttt{dohmatob@meta.com}
  \And
    Meyer Scetbon\\
  Microsoft Research$^{\dagger}$\\
  \texttt{t-mscetbon@microsoft.com}
}

\begin{document}
\maketitle



\begin{abstract}
In this paper, we investigate the impact of test-time adversarial attacks on linear regression models and determine the optimal level of robustness that any model can reach while maintaining a given level of standard predictive performance (accuracy). Through quantitative estimates, we uncover fundamental tradeoffs between adversarial robustness and accuracy in different regimes. We obtain a precise characterization which distinguishes between regimes where robustness is achievable without hurting standard accuracy and regimes where a tradeoff might be unavoidable. Our findings are empirically confirmed with simple experiments that represent a variety of settings. This work applies to feature covariance matrices and attack norms of any nature, and extends beyond previous works in this area.
\end{abstract}


\section{Introduction}
Machine learning models are known to be highly sensitive to small perturbations known as \emph{adversarial examples}~\citep{szegedy2013intriguing}, which are often imperceptible to humans.
While various strategies such as adversarial training~\citep{madry2017towards} can mitigate this vulnerability empirically, 
the situation remains highly problematic for many safety-critical applications like autonomous vehicles or health, and motivates a better theoretical understanding of what mechanisms may be causing this.

From a theoretical perspective, the case of classification is rather well-understood. \cite{tsipras18} showed that adversarial robustness could be at odds with accuracy. The hardness of classification under adversarial attacks has been crisply characterized \citep{bhagoji2019,bubeck2018}. In the special case of linear classification, explicit lower-bounds on sample complexity have been obtained \citep{schmidt2018,bhattacharjee2020sample}.

However, the case of regression is relatively understudied. Recently, in the setting of Euclidean-norm attacks on isotropic features, \cite{Javanmard2020PreciseTI} has initiated a theoretical study of a possible tradeoff between standard risk (a.k.a. generalization error) and adversarial risk (a.k.a. robust generalization error) for linear regression
with isotropic features,
where an adversary is allowed to attack the input data point at test-time. The authors computed exact Pareto optimal curves that reveal a tradeoff between standard and adversarial risk.


Our work mostly considers the following question:
\begin{restatable}{question}{}
\label{question:main}
In the context of linear regression, if a model has "small" standard risk, how "big" / "small" can its adversarial risk be ? Is possible to be robust while being accurate ? 
\end{restatable}
In the context of classification, an analogous question was considered in \cite{tsipras18}, where the authors constructed a high-dimensional problem for which every model with standard accuracy $1-\epsilon$ has adversarial accuracy at most $c\epsilon$, where $c$ is an absolute constant. Such a result is reminiscent of a tradeoff between standard generalization and robustness, and our results will have this flavor.

\paragraph{Summary of Our Contributions.}
The main contributions of this work precise quantitative estimates which allow us to distinguish between regimes where robustness is achievable without hurting standard accuracy and regimes where a tradeoff is might be unavoidable. Our main findings can be broken down as follows.
\begin{itemize}
\item \emph{Analytic Formula for Optimal Robustness.}
As a function of the attack strength, we obtain analytic estimates of the optimal adversarial risk.
Importantly, the model which attains optimal robustness is a regularized version of the generative model (the labelling function) with explicit regularization parameter. In the special case of Euclidean-norm attacks, it is a ridge estimator, and we recover a simplified formulation of the result obtained in~\cite{Xing2021}.
\item \emph{Free Lunch and Tradeoffs for Robustness.}
At any given attack strength, we establish a threshold on the standard risk above which no tradeoff between standard predictive performance (accuracy) and robustness is needed.
These results answer Question \ref{question:main} quantitatively.
Importantly, we show
that 
the model achieving the above accuracy / robustness tradeoff is a regularized estimate of the ground-truth / generative model, with regularization parameter given explicitly in terms of the attack strength and the accuracy tolerance.
\item \emph{Phase-Transition Diagrams.}
Our analytic results allow us to identify phase-transitions for robustness in different regimes. As concrete examples, we focus on two regimes: (i) the case of Euclidean-norm attacks on linear regression under polynomially-decaying spectral and source conditions, and (ii) the setting of $\ell_p$-norm attacks on distributions where the covariance matrix is isotropic, with various structural assumptions (e.g, sparsity) on the generative model. For both settings, we compute the complete phase-transition diagram illustrating the tradeoffs between standard accuracy and adversarial robustness.
\end{itemize}
Importantly, unlike previous works like \cite{Javanmard2020PreciseTI,Xing2021}, our analysis applies to general attack norms (not just Euclidean) and covariance matrices (not just isotropic).

\paragraph{Related works.}
The theoretical understanding of adversarial examples is now an active area of research.
Below is a list of works which are most relevant to our current paper. A detailed overview of the literature is discussed in the appendix / suppmat.

In the setting of classification, \cite{tsipras18} considers a specific data distribution where good accuracy implies poor robustness.
\citep{goldstein,saeed2018, gilmerspheres18,dohmatob19} show that for high-dimensional data distributions which have concentration property (e.g., multivariate Gaussians, distributions satisfying log-Sobolev inequalities, etc.), an imperfect classifier will admit adversarial examples. \cite{dobriban2020provable} studies tradeoffs in Gaussian mixture classification problems, highlighting the impact of class imbalance. On the other hand, \cite{closerlook2020} observed empirically that natural images are well-separated, and so locally-lipschitz classifiers should not suffer any kind of test error vs robustness tradeoff.

In the setting of linear regression (the setup considered in our work), \cite{Xing2021} studied Euclidean-norm attacks with general covariance matrices. They showed that the optimal robust model is a ridge regression whose ridge parameter depends implicitly on the strength of the attacks. \cite{Javanmard2020PreciseTI} studied tradeoffs between ordinary and adversarial risk in linear regression, and computed exact Pareto optimal curves in the case of Euclidean-norm attacks on isotropic features. Their results show a tradeoff between ordinary and adversarial risk for adversarial training. \cite{javanmard2021adversarial} also revisited this tradeoff for latent models and show that this tradeoff is mitigated when the data enjoys a low-dimensional structure. The analysis in \cite{Javanmard2020PreciseTI} is based on \emph{Gordon's Comparison Inequality} \cite{gordon88,Thrampoulidis15,Thrampoulidis18}, which is a very versatile tool in the analysis of regularized estimators but fails to produce analytic results when one deviates from the setting of Euclidean-norm attacks on isotropic features. In contrast, our analysis is based on basic Langrangian duality. It relies on some approximations which turn out to only introduce multiplicative absolute constants in the final result, but are completely harmless for the final analysis and interpretation.

The study of robustness in linear regression for general norms and feature covariance matrices has been initiated in \cite{MeyerAndDohmatob2023} which gave sufficient conditions for the generative model $w_0$ (and its estimators like gradient descent, ridge regression, etc.) to be robust. However, the the question of tradeoffs was not considered.

Finally, \cite{DohmatobAndBietti} established tradeoffs between accuracy and robustness to Euclidean-norm attacks on two-layer neural networks in different learning regimes.

\section{Problem Formulation}
\paragraph{Notaitons.} 
Given a positive-definite matrix $M$, the induced Mahahanobis norm $\|\cdot\|_M$ is define by $\|z\|_M := \|M^{1/2} z\|_2$.
The notation $f(d) = O(g(d))$; also written $f(d) \lesssim g(d)$, means  there exists an absolute constant $K$ such that $f(d) \le K\cdot g(d)$, perhaps for sufficiently large $d$. Likewise, $f(d)=\Omega(g(d))$ (or $f(d) \gtrsim g(d)$) means $g(d)=O(f(d))$. We write $f(d) \asymp g(d)$ (or $f(d) = \Theta(g(d))$) to mean $f(d) \lesssim g(d) \lesssim f(d)$. Finally, $f(d) = o(g(d))$ (or $f(d) \ll g(d)$) means $f(d)/g(d) \to 0$ perhaps for sufficientl large $d$. In particular, $f(d) = o(1)$ (or $f(d) \ll 1$) means that $f(d) \to 0$ perhaps for sufficiently large $d$.

\subsection{Data Distribution}
In this work, we consider linear regression problem given by the following distribution $P$ over a $d$-dimensional feature vector $x \in \mathbb R^d$ and labels $y \in \mathbb R$
\begin{eqnarray}
\label{eq:generator}
\begin{split}
&\textbf{(Features) }x  \sim P_{x } := N(0,\Sigma),\\
&\textbf{(Label) }y  = x^\top w_0 + z,\text{ with }z \sim N(0,\sigma^2),\text{ independent of }x.
\end{split}
\end{eqnarray}

Thus, the marginal distribution $P_x$ of the features is a multivariate Gaussian with $d \times d$ positive-definite covariance matrix $\Sigma$. The generative model for the labels is a linear model defined by $x \mapsto f_{w_
0}(x):=x^\top w_0$, for some fixed vector of parameters $w_0 \in \mathbb R^d$. To avoid trivialities, we will assume WLOG that $w_0 \ne 0$. The scalar $\sigma \ge 0$ is the strength of the label-noise $z  \sim N(0,\sigma^2)$. The input-dimension $d$ is not assumed fixed, and in fact, for the better part of this paper, we shall consider phenomena happening in the limit $d \to \infty$. One should keep in mind that in such a case, we are actually considering a sequence of problems (i.e. distributions $P(d)$) indexed by $d$.

Such a data distribution is also the setup of previous works like \cite{Javanmard2020PreciseTI,Xing2021,MeyerAndDohmatob2023}. Note however that in \cite{Javanmard2020PreciseTI}, the covariance matrix is trivial / isotropic, i.e. $\Sigma=I_d$. In contrast, as in \cite{MeyerAndDohmatob2023}, our work considers general covariance matrices $\Sigma$.

\subsection{Linear Models, Standard and Adversarial Risks}
This work considers regression over linear models $f_w(x):=x^\top w$, parametrized by a weights vector $w \in \mathbb R^d$. An adversarial attack replaces a clean data point $(x,y) \sim P$ by a perturbed version $(x+\delta,y)$. The size of the perturbation $\delta=\delta(x,y)$ is measured w.r.t a pre-specified norm $\|\cdot\|$ on the feature space $\mathbb R^d$. Note that the attacker is only allowed to change the feature vector $x$, and not the label $y$. By an attack of strength $r \ge 0$, we mean that the constraint $\|\delta\| \le r$ is enforced. The goal of the attacker is to make the prediction $f_w(x+\delta)$ on the corrupt feature vector $x+\delta$ deviate from the ground-truth label $y$ of clean feature vector $x$, as much as possible.


\begin{restatable}[Risks]{df}{}
Given $w \in \mathbb R^d$, attack budget $r \ge 0$ w.r.t a arbitrary norm $\|\cdot\|$ on $\mathbb R^d$, the adversarial risk (a.k.a adversarial generalization error) of the linear model $f_w$ is defined by
\begin{equation}
E(w,r)  = E^{\|\cdot\|}(w,r) := \mathbb E\left[\sup_{\|\delta\| \le r}(f_w(x+\delta)-y)^2\right]\text{ with }(x,y) \sim P.
\label{eq:Ewr}
\end{equation}
Also, recall the definition of the standard risk (a.k.a standard generalization error) of $f_w$, namely
\begin{equation}
    E(w) :=  \mathbb E[(f_w(x)- y)^2]  =   \|w-w_0\|_\Sigma^2 + \sigma^2.
\end{equation}
\end{restatable}
Of course, $E(w,r) \ge E(w,0) = E(w)$ for any $w \in\mathbb R^d$ and $r \ge 0$, with equality if $r=0$.  

\begin{restatable}{rmk}{}
Note that by definition, $E(w,r)$ depends on the attacker's norm $\|\cdot\|$. To simplify the notations, we will omit its dependency in the following and precise the norm if needed.
\end{restatable}

Let $\|\cdot\|_\star$ be the dual of  $\|\cdot\|$, defined by $\|w\|_\star := \sup_{\|\delta\| \le 1} \delta^\top w$.
The following lemma established in \cite{MeyerAndDohmatob2023} (also see \cite{Xing2021,javanmard2022precise} for the special case of Euclidean-norm attacks) gives an analytic formula for the adversarial risk which will be very useful in the sequel.
\begin{restatable}{lm}{analytic}
\label{lm:analytic}
$E(w,r) = E(w)+ r^2\|w\|_\star^2 + 2\sqrt{2/\pi} r\|w\|_\star\sqrt{E(w)}$ for any $w \in \mathbb R^d$ and $r \ge 0$.
\end{restatable}

\paragraph{Genuinely small / Imperceptible Adversarial Attacks.}
In practice, one is usually only concerned with adversarial attacks which are \emph{genuinely small} to the eye. Formally, this means that the attack strength $r$ is restricted to be much smaller than the average norm of a random data point, i.e. 
\begin{eqnarray}
\label{eq:genuinely small}
r = o_d(R(\Sigma)),\text{ with } R(\Sigma):= \mathbb E\,\|x\|\text{ for }x \sim P_x = N(0,\Sigma).
\end{eqnarray}
For example, in the case of Euclidean-norm attack on isotropic features with covariance matrix $\Sigma=I_d$, we have $R(\Sigma) = R(I_d) \asymp \sqrt{\mbox{tr}(\Sigma)} \asymp \sqrt d$ in the limit $d \to \infty$. Thus, in this case, an attack is only genuinely small in the sense of the above definition iff $r/\sqrt d = o(1)$. For example, $r = \sqrt{d/\log d}$ would be genuinely small. On the other hand, if $\Sigma=(1/d)I_d$, then $R(\Sigma) \asymp \sqrt{\mbox{tr}(I_d)/d} = 1$. Thus, in this case, a Euclidean-norm attack of size $r=r(d)$ would be genuinely small iff $r = o(1)$.


\section{Analysis of the Optimal Robustness}
\begin{restatable}
{df}{}
Given an attack strength $r \ge 0$, let  $E_{opt}(r)$ be the optimal adversarial risk,
\begin{eqnarray}
E_{opt}(r) := \min_{w \in \mathbb R^d} E(w,r).
\end{eqnarray}
Furthermore, let $w_{opt}(r)$ denote any $w \in \mathbb R^d$ which achieves this optimum.
\end{restatable}

Observe that the expression for adversarial risk $E(w,r)$ given in Lemma \ref{lm:analytic} exhibits a tension between the standard risk $E(w)$, which is minimized by the generative model $w_0$, and the dual norm $\|w\|_\star$, which is minimized by the null model $w=0$. Thus, for a given attack strength $r$, one would expect that the optimal robust model $w_{opt}(r)$ would have to somehow interpolate between $w_0$ and $0$.
In this section, we show that this is indeed the case (Theorem \ref{thm:Eopt}).

\subsection{Adversarial Risk Proxies}
Even though the adversarial risk functional $E$ admits an analytic formula thanks to Lemma \ref{lm:analytic}, that expression does not lend itself well to analysis. Following \cite{MeyerAndDohmatob2023}, we shall resort to multiplicative approximations defined as follows.
 For any linear model $w \in \mathbb R^d$ and attack strength $r \ge 0$, set
\begin{align}
\overline E(w,r) & = \overline E^{\|\cdot\|}(w,r) := \sigma^2 + \|w-w_0\|_\Sigma^2 + r^2\|w\|_\star^2,\\
\widetilde E(w,r) &= \widetilde E^{\|\cdot\|}(w,r) := \sigma^2 + K(w,r)^2, \text{ with }K(w,r):= \|w-w_0\|_\Sigma + r\|w\|_\star.
\end{align}
The following lemma
shows that $\overline E(w,r)$ and $\widetilde E(w,r)$ are indeed proxies (i.e. multiplicative approximations) of the adversarial risk $E(w,r)$. 
\begin{restatable}{lm}{proxy}
\label{lm:proxy}
There exists absolute constants $c_1$ and $c_2$ such that for a general attacker norm $\|\cdot\|$,
\begin{eqnarray}
\widetilde E(w,r) \le E(w,r) \le c_1\widetilde E(w,r)\text{ and }\overline E(w,r) \le E(w,r) \le c_2\overline E(w,r),\text{ for all }w \in \mathbb R^d,\, r \ge 0.
\end{eqnarray}
\end{restatable}
 The first part was established in \cite{MeyerAndDohmatob2023}. The second part follows similar arguments as the first one. See appendix for the full proof, including explicit values $c_1$ and $c_2$.
\begin{restatable}{rmk}{}
The multiplicative approximations given in Lemma~\eqref{lm:proxy} will be sufficient for our purposes whereby we will only be interested in understanding the orders of magnitude of the adversarial risk of models relative to the optimum value $E_{opt}(r)$, as a function of the attack strength $r$.
\end{restatable}

\subsection{Robustness via Regularization}
For any $\lambda \ge 0$, let $w^{prox}(\lambda)$ be the unique minimizer of $\overline E(w,\sqrt \lambda)$ over $w \in \mathbb R^d$. Thus,
\begin{eqnarray}
\label{eq:wt}
w^{prox}(\lambda) := \arg\min_{w \in \mathbb R^d} \|w-w_0\|_\Sigma^2 + \lambda\|w\|_\star^2.
\end{eqnarray}
Thus, $w^{prox}$ is the \textit{proximal operator} w.r.t the squared-Mahalanobis norm $\|\cdot\|_\Sigma^2$, of the square of the dual $\lambda \|\cdot\|_\star^2$ of the attacker's norm, evaluated at the point $w_0$. Of course, it implicitly depends on the choice of the norm $\|\cdot\|$ of the attacker. For example, in the special case of Mahalanobis-norm attacks w.r.t. any positive definite matrix $B$, that is when $\|\cdot\| = \|\cdot\|_B$, we have the closed-form expression $w^{prox}(\lambda) = (\Sigma + \lambda B^{-1})^{-1} \Sigma w_0$.
The structure of $w^{prox}(\lambda)$ in the case of more general norms (e.g $\ell_p$, etc.) is discussed in Appendix \ref{sec:structure}.

Define auxiliary functions $G:\mathbb R_+ \to \mathbb R_+$ and $F:\mathbb R_+^2 \to \mathbb R_+$ by
\begin{align}
    G(\lambda) &= G^{\|\cdot\|}(\lambda) := \|w^{prox}(\lambda)-w_0\|_\Sigma^2,\,F(r,\lambda) = F^{\|\cdot\|}(r,\lambda) := G(\lambda) + r^2\|w^{prox}(\lambda)\|_\star^2\label{eq:FG},
\end{align}
The following result which holds for any choice of the attacker's norm $\|\cdot\|$ is one of our main results.
\begin{restatable}{thm}{Eopt}
\label{thm:Eopt}
With $\lambda = r^2$, it holds that $E_{opt}(r) \asymp E(w^{prox}(\lambda),r) \asymp \sigma^2 + F(r,r^2)$. That is, up to within multiplicative absolute constants, $w^{prox}(\lambda=r^2)$ attains the optimal adversarial risk $E_{opt}(r)$.
\end{restatable}
Note that the above result is valid for any attacker norm. The special case of Euclidean-norm attacks was handled in \cite{Xing2021} where it was shown that $w_{opt}(r) = (\Sigma + \lambda I_d)^{-1}\Sigma w_0$, for some $\lambda\in [0, \infty]$ which depends on $r$, $w_0$, and $\Sigma$, via a fixed-point equation. Even, in this scenario, our result above gives a much clearer understanding, since it proposes to use the explicit ridge parameter $\lambda=r^2$, which clearly highlights the the dependence on the attack strength $r$.


\section{
Characterizing  Accuracy Tradeoffs for Robustness}
\label{sec:lagrangian}
\subsection{Preliminaries}
\label{subsec:goal}
In view of addressing Question \ref{question:main}, our objective is to compute bounds on the optimal adversarial risk (with respect to any given norm) over all linear models which attain a certain level of standard risk. 
For any linear model $w \in \mathbb R^d$, let $\Delta(w) := (E(w)-\sigma^2)/\|w_0\|_\Sigma^2$ be the (normalized) excess standard risk of $w$.
The division by $\|w_0\|_\Sigma^2$ ensures that  $\Delta(w_0) = 0$ while $\Delta(0) = 1$.
\paragraph{Optimal Robustness of Accurate Models.}
For any $r,\epsilon \ge 0$, let $\mathcal W_\epsilon$ be the set of all $\epsilon$-accurate models, i.e.
\begin{eqnarray}
\label{eq:Weps}
\mathcal W_\epsilon := \{w \in \mathbb R^d \mid \Delta(w) \le \epsilon^2\} = \{w \in \mathbb R^d \mid \|w-w_0\|_\Sigma \le \epsilon\|w_0\|_\Sigma\},
\end{eqnarray}
and let $E_{opt}(r,\epsilon)$ be the optimal adversarial risk of such models against attacks of strength $r$, i.e.
\begin{eqnarray}
\label{eq:Eopteps}
    E_{opt}(r,\epsilon) := \min_{w \in \mathcal W_\epsilon}E(w,r).
\end{eqnarray}
Finally, let $w_{opt}(r,\epsilon)$ denote any $w \in \mathbb R^d$ which achieves the above optimum. $E_{opt}(r,\epsilon)$ will be the main object of study of our paper as it captures the sacrifice (if any!) in robust that must be made by accurate models. 

The following lemma shows that this constrained formulation of the adversarial risk minimization problem can only differ from the unconstrained one when $\epsilon \in [0,1)$.
\begin{restatable}{lm}{bigeps}
\label{lm:bigeps}
For any $r \ge 0$ and $\epsilon \ge 1$, it holds that $E_{opt}(r,\epsilon) = E_{opt}(r)$. 
\end{restatable}

\paragraph{Free Lunch Threshold.} Given any $r \ge 0$, the quantity $\epsilon_{FL}(r)  \in [0,1]$ defined by
\begin{eqnarray}
    \epsilon_{FL}(r) := \sqrt{\Delta(w^{prox}(r^2))} = \sqrt{ G(r^2)}/\|w_0\|_\Sigma,
    \label{eq:threshold}
\end{eqnarray}
will be called the "free lunch" threshold for attacks of strength $r$, a terminology that will become clear shortly in Theorem \ref{thm:freelunch}. In addition we also need to introduce the following lemma.

\begin{restatable}{lm}{implicit}
\label{lm:implicit}
For any $r \ge 0$ and $0 \le \epsilon \le \epsilon_{FL}(r)$, the scalar equation
\begin{eqnarray}
    G(\lambda) = \epsilon^2\|w_0\|_\Sigma^2
    \label{eq:implicit}
\end{eqnarray}
has a unique solution $\lambda_{opt}(r,\epsilon)$ in $[0,r^2]$. 
\label{lm:decreasing}
We also extend the definition of $\lambda_{opt}(r,\epsilon)$ to all $\epsilon \in [0,1]$ by setting $\lambda_{opt}(r,\epsilon) = r^2$ whenever $\epsilon \ge \epsilon_{FL}(r)$.
\end{restatable}


\subsection{Main Result}
We are now ready to present the main result of this work and show in the following theorem tradeoffs between standard accuracy and adversarial robustness in the setting of general feature covariance matrix $\Sigma$ and attacker norm $\|\cdot\|$.




 
\begin{restatable}{thm}{freelunch}
For any attack strength $r \ge 0$ and tolerance $\epsilon \in [0,1]$, the following hold.

(A) (\textbf{Accuracy vs Robustness Tradeoff}) It holds that 
\begin{eqnarray}
    E_{opt}(r,\epsilon) \asymp E(w^{prox}(\lambda_{opt}(r,\epsilon)),r) \asymp \sigma^2 + F(r,\lambda_{opt}(r,\epsilon)),
\end{eqnarray}
where $\lambda_{opt}(r,\epsilon) \in [0,r^2]$ is as in Lemma \ref{lm:decreasing} and $\lambda \mapsto w^{prox}(\lambda)$ is as defined in \eqref{eq:wt}.
That is, up to within multiplicative absolute constants, with the choice $\lambda=\lambda_{opt}(r,\epsilon)$ the vector  $w^{prox}(\lambda)$ attains the optimal adversarial risk $E_{opt}(r,\epsilon)$ over all $\epsilon$-accurate models.

 (B) (\textbf{Free Lunch}) If $\epsilon \ge \epsilon_{FL}(r)$, then it holds that
$E_{opt}(r,\epsilon) \asymp E_{opt}(r)$.
That is, no accuracy / robustness tradeoff is needed when the excess risk level $\epsilon$ is greater than the threshold $\epsilon_{FL}(r)$: there is always an $\epsilon$-accurate model which achieves the absolute optimal (up to within multiplicative absolute constants) adversarial risk $E_{opt}(r)$.
\label{thm:freelunch}
\end{restatable}



% ...

\section{Some Consequences of Our Results}
\label{sec:consequences}
We now apply our Theorems \ref{thm:Eopt} and \ref{thm:freelunch}, to obtain some concrete consequences in a variety of settings. Section \ref{sec:experiments} will provide some empirical confirmation of these predicted consequences.



\subsection{Sparsity with Isotropic Features}
Consider the case where the feature covariance matrix $\Sigma$ and the generative model $w_0$ are given by
\begin{eqnarray}
\Sigma=(1/d)I_d,\,w_1=\ldots=w_s = 1,\,w_{s+1} = \ldots = w_d=0,
\label{eq:sparse-isotropic}
\end{eqnarray}
for some sparsity parameter $s \in [d]$. We consider $\ell_p$-norm attacks, for some fixed $p \in [1,\infty]$.

First observe, that for $\ell_\infty$-norm attacks of strength $r$, the adversarial risk of the generative mode $w_0$ is given by $E(w_0,r) = \sigma^2 + r^2\|w_0\|_1^2 = \sigma^2 + s^2r^2$, while for $\ell_p$-norm attacks with $p \in [1,\infty)$,  we have $E(w_0,r) = \sigma^2 + r^2\|w_0\|_q^2 = \sigma^2 + r^2 s^{2/q}$ where $q :=  \in [1,\infty]$ is the harmonic conjugate of $p$.

\begin{restatable}{thm}{isotropicEps}
Recall the notations of Theorem \ref{thm:freelunch}. Let the attack norm be an $\ell_p$ with $p \in [1,\infty]$. For any $r \ge 0$ and $\epsilon \in [0,1]$, the robustness profile is given as in Table \ref{tab:isotropic-eps}.

In particular, in the limit $d \to \infty$, we have that, if
\begin{itemize}
\item[--] $p \in [1,\infty)$, $1 \ll s \le d$, and we take $r \asymp 1/s^{1/q}$, OR
\item[--] $p=\infty$, $\sqrt{d/\log d} \ll s \le d$, and we take $r \asymp 1/s$, 
\end{itemize}
then for $\epsilon \in [0,1)$, it holds that
\begin{eqnarray}
   E_{opt}(r) = o(1),\, E_{opt}(r,\epsilon) =  \Theta((1-\epsilon)^2).
   \label{eq:dust}
\end{eqnarray}
\label{thm:isotropic-eps-different-p}
\end{restatable}
\begin{table}[!h]
\begin{center}
\begin{tabular}{c|c|c|c|c} 
 \hline
 & $\epsilon_{FL}(r)$  & $\lambda_{opt}(r,\epsilon)$ & $E_{opt}(r)$ & $E_{opt}(r,\epsilon)$\\
\hline
$p=2$ & $r^2/(1+r^2)$ & $\epsilon/(1-\epsilon)$ & $\sigma^2 + (s/d)\min(r\sqrt{d},1)^2$ & $\sigma^2 + (s/d)H(r\sqrt d,\epsilon)^2$\\
$p \ne 2$ &  -- & -- & $\sigma^2 + (s/d)\min(r/r_0(p),1)^2$ & $\sigma^2 + (s/d)H(r/r_0(p),\epsilon)^2$\\
\hline
\end{tabular}
\end{center}
    \caption{Details of Theorem \ref{thm:isotropic-eps-different-p}. Here, $r_0(p) = s^{1/p-1/2}/\sqrt d$. In particular, $r_0(2) = 1/\sqrt d$, $r_0(\infty) = 1/\sqrt{sd}$. The function $H$ is defined by $H(r,\epsilon) = r$ if $r \le 1$; else $H(r,\epsilon) = \epsilon + (1-\epsilon)r$.}
\label{tab:isotropic-eps}
\end{table}

\subsection{Polynomial Spectral Decay}
\label{subsec:polydecay}
Let $\Sigma = \sum_{k \ge 1} \lambda_k \phi_k\phi_k^\top$ be the spectral decomposition of the feature covariance matrix $\Sigma$ and $c_k = \phi_k^\top w_0$ be the $k$-th alignment coefficient of the generative model $w_0$, so that $w_0 = \sum_{k \ge 1} c_k \phi_k$. 
We place ourselves in the high-dimensional setting ($d \to \infty$), and assume spectral information $(\lambda_k,c_k)_{k \ge 1}$ is given by the following polynomial (aka power-law) scalings
\begin{eqnarray}
\lambda_k \asymp k^{-\beta},\, \text{ and }c_k^2 \asymp k^{-\delta}\text{ for all }k,
\label{eq:polyregime}
\end{eqnarray}
where $\beta > 1$ and $\delta \ge 0$ are constants. This model is well-studied in the literature \cite{caponnetto2007,justinterpolate} because (1) it usually leads to tractable analysis, and (2) it can be used to approximate the the macroscopic structure of certain neural networks in the kernel regime (i.e. wide neural networks) \cite{Bahri2021,Cui2022,Wei2022MoreThanAToy}.
In this setting, observe that $\mbox{tr}(\Sigma) \asymp \sum_k k^{-\beta} = \Theta(1)$, $\|w_0\|_\Sigma^2  \asymp \sum_k k^{-\beta-\delta} = \Theta(1)$, while
\begin{eqnarray}
\begin{split}
\|w_0\|_2^2 &\asymp \sum_k k^{-\delta}
\asymp \begin{cases}
d^{1-\delta},&\mbox{ if }0 \le \delta < 1,\\
\log d,&\mbox{ if }\delta= 1,\\
1,&\mbox{ if }\delta > 1.
\end{cases}
\end{split}
\end{eqnarray}
Also note that for Euclidean-norm attacks, we have $R(\Sigma) \asymp \sqrt{\mathrm{tr}(\Sigma)} = \Theta(1)$.

The following result is one of our main contributions.
\begin{restatable}{thm}{polydecay}
For Euclidean-norm attacks of small strength $r  > 
 0$, the conclusions of Theorem \ref{thm:freelunch} prevail, and the quantities $\epsilon_{FL}(r)$, $\lambda_{opt}(r,\epsilon)$, $E_{opt}(r)$, and $E_{opt}(r,\epsilon)$ are as given in Table \ref{tab:poly}. Thus, as regards robustness, $\delta=1$ is a critical value for the source exponent in \eqref{eq:polyregime}: For $\delta \in [0,1]$, accuracy (controlled by the excess risk tolence $\epsilon$) has to be traded for robustness, while for $\delta \in (1,\infty)$, the generative model $w_0$ is so smooth that robustness and accuracy are aligned.


Consider the particular regime where $0 \le \delta \le 1$. For small $\sigma^2 \ge 0$, $\epsilon  > 0$, and $r=r(\epsilon)$ given by
\begin{eqnarray}
r = \begin{cases}
\epsilon^\phi,&\mbox{ if }0 \le \delta < 1,\\ \sqrt{1/\log(1/\epsilon)},&\mbox{ if }\delta=1
\end{cases}
\end{eqnarray}
with $\theta := (1-\delta)/\beta \ge 0$ and $\phi := \theta/(1-\theta) \ge 0$, it holds that
\begin{eqnarray}
    E_{opt}(r) = o(1),\, E_{opt}(r,\epsilon) = \Theta(1).
\end{eqnarray}
That is, even though robustness to imperceptible attacks is achievable in this setting, accurate models (especially the generative model $w_0$ itself) are non-robust.
\label{thm:polydecay}
\end{restatable}
\begin{table}[!h]
\begin{center}
\begin{tabular}{c|c|c|c|c|c} 
 \hline
Regime & $\epsilon_{FL}(r)$  & $\lambda_{opt}(r,\epsilon)$ & $E_{opt}(r)$ & $E_{opt}(r,\epsilon)$ & FL ?\\
\hline
$0 \le \delta < 1$ & $r^{2(1-\theta)}$ & $\epsilon^{2/(1-\theta)}$ & $\sigma^2 + r^{2(1-\theta)}$ & $\sigma^2 + \epsilon^2 + r^2\epsilon^{-2\phi}$ & No \\
$\delta=1$ & $r^4\log(1/r)$ & $e^{W(-\Theta(\epsilon^2))/2}$ & $\sigma^2 + r^2\log(1/r)$ & $\sigma^2 + \epsilon^2 + r^2\log(1/\epsilon)$ & No \\
$\delta >1$ & $r^4$ & $\epsilon$ & $\sigma^2 + r^2$ & $\sigma^2 + \epsilon^2 + r^2$ & Yes!\\
\hline
\end{tabular}
\end{center}
\caption{Details for Theorem \ref{thm:polydecay}. Here, $W$ is an appropriate branch of the Lambert function. Note that except for the first column, all the entries in the table are given only within multiplicative absolute constants.
The last column records whether there is free lunch (FL), wherein robustness is achievable without sacrificing accuracy.
}
\label{tab:poly}
\end{table}

Thus, there is a phase-transition at $\delta=1$ whereby accurate models (including generative model $w_0$, i.e. the case $\epsilon=0$) switch from non-robust to robust. This is also empirically confirmed in Section \ref{sec:experiments}.

\subsection{A Non-Euclidean Setting: $\ell_\infty$-Norm Attacks on Isotropic Features}
Let us now present an example of non-Euclidean scenario where generative model $w_0$ fails to be robust to genuinely small adversarial perturbations.
Still in high dimensions ($d \to \infty$), consider the setting where the feature covariance matrix is $\Sigma=I_d$ while the coefficients of the generative model have the following "harmonic" distribution $w_0$
\begin{eqnarray}
(w_0)_k=1/k,\text{ for all }k \in [d].
\label{eq:power2}
\end{eqnarray}
For $\ell_\infty$-norm attacks, note that $R(\Sigma) \asymp \sqrt{\log(\mathrm{tr}(\Sigma))} \asymp \sqrt{\log d}$. We have the following.
\begin{restatable}
{thm}{peetre}
\label{thm:peetre}
In the limit $d \to \infty$, it holds for $\ell_\infty$-norm attacks of strength $r$ with $1/\sqrt d \le r = o(1)$ that $E_{opt}(r) \asymp \sigma^2 + r^2 \log (1/r)^2$.

In particular, for $r \asymp 1/\log d$ and $\sigma^2=o(1)$, it holds that
\begin{align}
E(w_0,r) = \Theta(1),\, E_{opt}(r) = o(1).
\end{align}
That is, even though robustness is achievable, the generative model $w_0$ is itself non-robust.
\end{restatable}

\section{Empirical Verification}
\label{sec:experiments}
We provide a series of simple experiments on simulated data to empirically verify our theoretical results. Additional experiments are provided in the suppmat / appendix.

\textbf{Experiment 1 (Verification of Theorem \ref{thm:isotropic-eps-different-p}).}
For this experiment, we fix the input dimension $d=400$, while the covariance matrix $\Sigma$ and generative model $w_0$ are as in \eqref{eq:sparse-isotropic} for different values of the sparsity parameter $s \in \{10, 20, d=400\}$. For different values of sample size $n$ from $d$ to $10^4$, we generate (5 runs) an iid dataset $\mathcal D_n = \{(x_1,y_1),\ldots,(x_n,y_n)\}$ and construct a ordinary least-squares (OLS) estimate $\widehat w_n$ for $w_0$. Note that this estimator is known to be consistent in the regime considered. Next, we compute the excess standard risk of $\widehat w_n$, namely $\epsilon = \epsilon_n := (\|\widehat w_n - w_0\|_\Sigma^2 - \sigma^2)/\|w_0\|_\Sigma^2$. We consider $\ell_p$-norm attacks with $p \in \{2,\infty\}$. The attack strength $r$ is set as in the second part of Theorem \ref{thm:isotropic-eps-different-p}. We compute the adversarial risk of $\widehat w_n$, alongside the adversarial risk of the generative model $w_0$, via the formula given in Lemma \ref{lm:analytic}.
The results for the experiment are shown in Figure \ref{fig:isotropic-Eopteps}. From the figure, we clearly see that the predictions of Theorem \ref{thm:isotropic-eps-different-p} are confirmed.

\begin{figure}[!h]
    \centering
     \begin{subfigure}{(a) Euclidean-norm (i.e. $p=2$) attack. Here we take $r  = 1 / \sqrt s$.}{}
         \centering
    \includegraphics[width=1\textwidth]{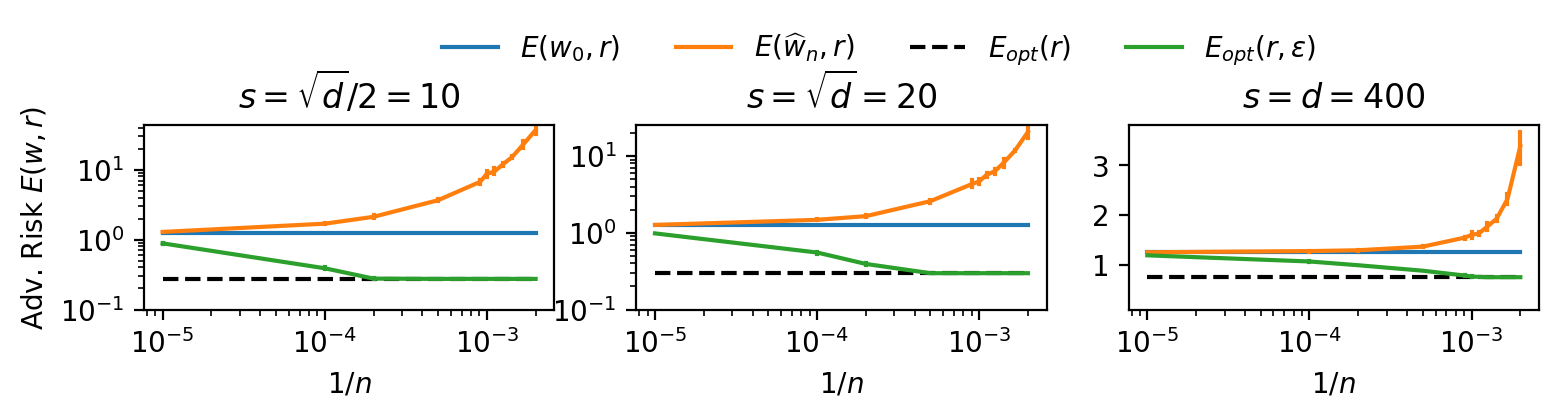}
    \end{subfigure}
     \begin{subfigure}{(b) $\ell_\infty$-norm attack. Here we take $r  = 1 / s$.}{}
    \includegraphics[width=1\textwidth]{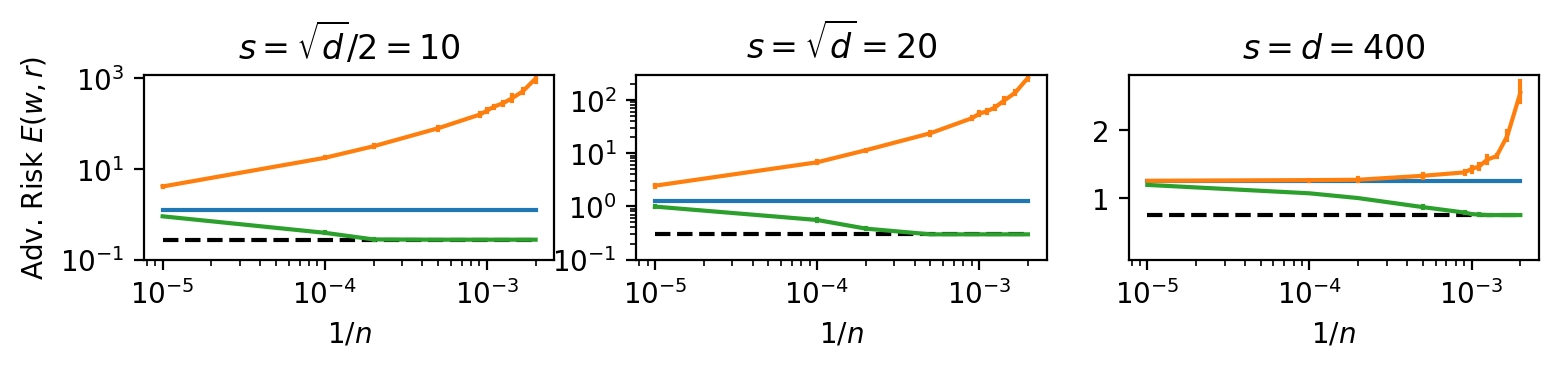}
\end{subfigure}
\vspace{-.6cm}
    \caption{ (\textbf{Experiment 1}) Empirical verification of Theorem \ref{thm:isotropic-eps-different-p}, for different levels of sparsity $s$ of the generative model $w_0$. Here the input dimension is set to $d=400$ and $n$ is the sample size. The theoretical curves $E_{opt}(r)$ and $E_{opt}(r,\epsilon)$ are as given by the theorem. Error bars correspond to 5 different runs of computing $\widehat w_n$ (OLS). Notice the conformity with the theorem's predictions.}
    \label{fig:isotropic-Eopteps}
\end{figure}

\textbf{Experiment 2 (Verification of Theorem \ref{thm:polydecay}).}
The setup for this experiment is as in \textbf{Experiment 1}, but with input dimension $d=10^4$, feature covariance matrix $\Sigma$ and generative model $w_0$ given as in \eqref{eq:polyregime}. We consider Euclidean-norm attacks with strength $r$ as given in Theorem \ref{thm:polydecay}.

The results for the experiment are shown in Figure \ref{fig:poly-decay-eps}. As predicted by the theorem, we see that for $\delta \in [0, 1]$, the ground-truth model $w_0$ is non-robust. Furthermore, for $\delta > 1$, $w_0$ becomes robust to small adversarial perturbations (blue curve and broken black line coincide) as predicted.
As $n \to \infty$ (i.e. $1/n \to 0$), the adversarial risk $E(\widehat w_n,r)$ of the estimator $\widehat w_n$ approaches that of the ground-truth model, namely $E(w_0,r)$; we see from the figure that  is optimal in the smooth regime where $\delta > 1$, but catastrophic in the non-smooth regime ($\delta \in [0,1]$), in conformity the theorem.

We also consider the case $\epsilon=0$, corresponding the ground-truth model, and vary the attack strength $r$. The results are shown in Figure \ref{fig:poly-Eopt}. Here again, the predictions of Theorem \ref{thm:polydecay} are confirmed.


\begin{figure}[!h]
    \centering
    \includegraphics[width=1\textwidth]{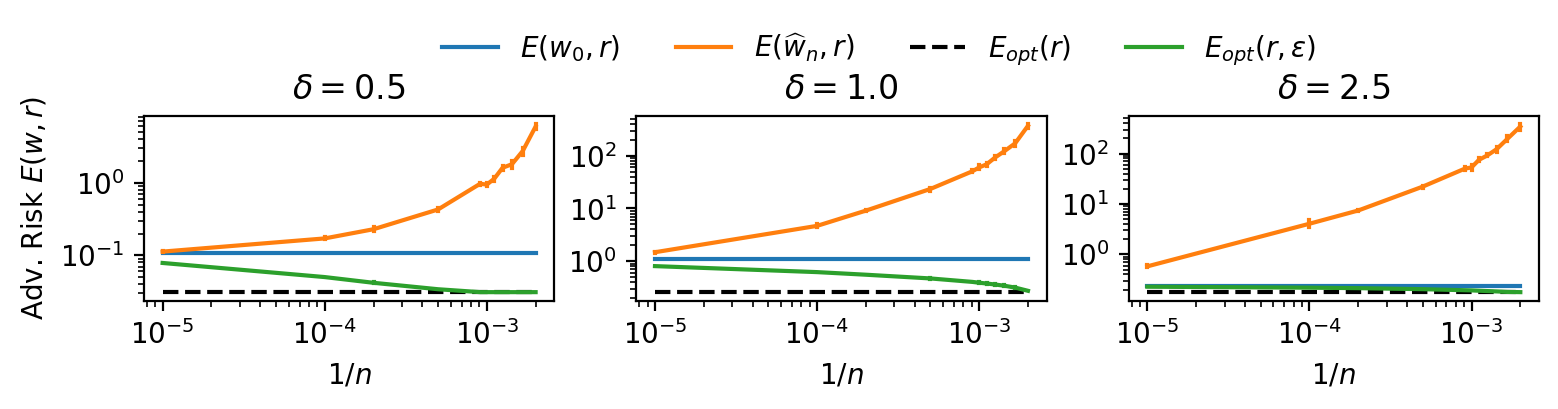}
\vspace{-.6cm}
    \caption{(\textbf{Experiment 2}) Empirical verification of Theorem \ref{thm:polydecay}. Here $\beta=2$ and $d=10^4$, while $n$ is the sample size. Error bars correspond to 5 different runs of computing the OLS estimator $\widehat w_n$. The curves $E_{opt}(r)$ and $E_{opt}(r,\epsilon)$ are as given by the theorem. Notice the conformity of the results with the theorem.
    }
    \label{fig:poly-decay-eps}
\end{figure}
\begin{figure}[!hbt]
    \centering
    \includegraphics[width=1\textwidth]{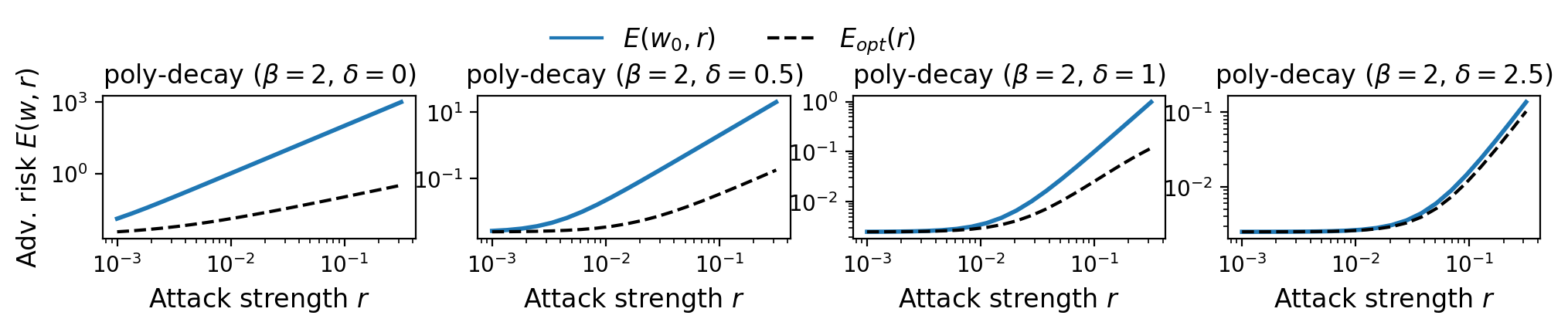}
    \caption{\textbf{(Experiment 2)} Empirical validation of Theorem \ref{thm:polydecay} for the case when $\epsilon=0$.  Notice the conformity of the results with the theorem.}
    \label{fig:poly-Eopt}
\end{figure}

\textbf{Experiment 3: Non-Robustness of Generative Model $w_0$.}
This experiment is meant to verify Theorem \ref{thm:peetre} and Theorem \ref{thm:isotropic-eps-different-p} for the case $\epsilon=0$ (corresponding to the generative model $w_0$). We fix the input dimension to $d=400$ as \textbf{Experiment 1}, and consider 3 different scenarios for the attacker's norm $\|\cdot\|$, the generative model $w_0$, and the feature covariance matrix $\Sigma$.
\begin{itemize}
\item \textbf{Experiment 3(a)}: Here, the feature covariance matrix is $\Sigma=(1/d)I_d$ and the generative model is $w_0 = 1_d = (1,\ldots,1)$. The attacker's norm is Euclidean, i.e. $p=2$ in Theorem \ref{thm:isotropic-eps-different-p}
\item \textbf{Experiment 3(b)}: Here, $\Sigma=(1/d)I_d$ and  $w_0$ is as in \eqref{eq:sparse-isotropic} with $s=20$. The attacker's norm is $\ell_\infty$, i.e. $p=\infty$ in Theorem \ref{thm:isotropic-eps-different-p}.
\item \textbf{Experiment 3(c)}: Here, $\Sigma$ and $w_0$ are as in Theorem \ref{thm:peetre} and the attacker's norm is $\ell_\infty$.
\end{itemize}
For different values of attack strength $r$, we compute the adversarial risk $E(w_0,r)$ of the generative model $w_0$ and compare it to the optimum adversarial risk $E_{opt}(r)$. The results of the experiment are shown in Figure \ref{fig:simple}. We see that the predictions of Theorem \ref{thm:isotropic-eps-different-p} (\textbf{Left} and \textbf{Middle} plots) and Theorem \ref{thm:peetre} (\textbf{Right} plot) are perfectly confirmed.

\begin{figure}[!hbt]
    \centering
    \includegraphics[width=1\textwidth]{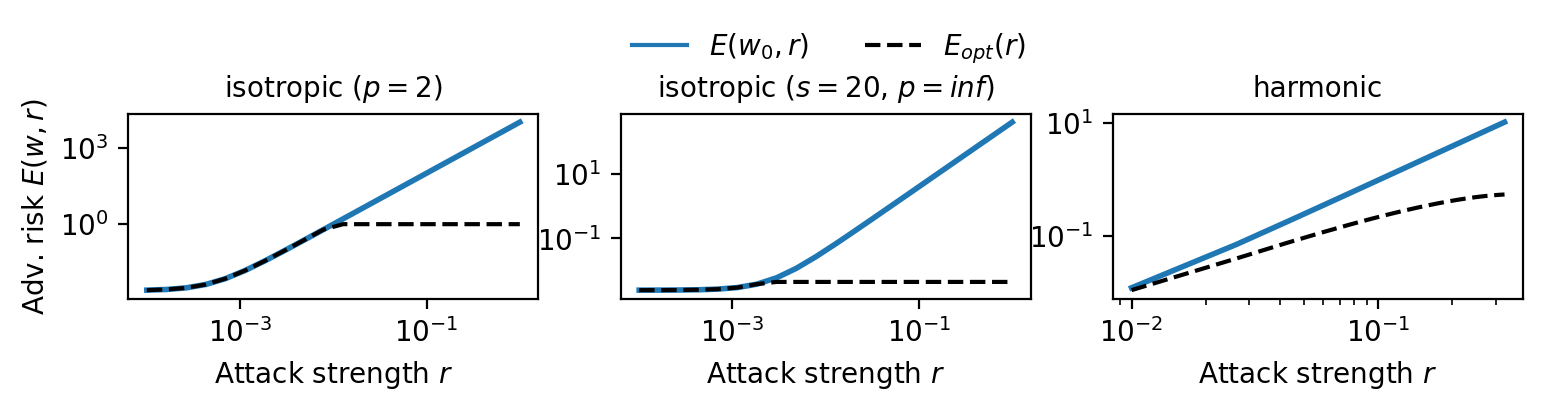}
    \vspace{-.5cm}
    \caption{Failure of ground-truth model  $w_0$ to be robust to small adversarial perturbations. From left to right, the plots show results for \textbf{Experiment 3(a--c)}. Here, we see perfect confirmation of Theorem \ref{thm:isotropic-eps-different-p} for $\epsilon=0$ and $p \in \{2,\infty\}$ (\textbf{Left} and \textbf{Middle} plots), and Theorem \ref{thm:peetre} (\textbf{Right} plot).}
    \label{fig:simple}
\end{figure}

\section{Concluding Remarks}
In this work, we have considered the problem of adversarial robustness in linear regression and have obtained precise quantitative estimates that allow us to uncover fundamental tradeoffs between adversarial robustness and standard accuracy in different regimes. Unlike previous works, our results apply to arbitrary covariance structures and attack norms.

\textbf{Possible Extensions.}
An interesting future direction of our work will be to extend the scope to neural networks in linearized regimes like random features. As in \cite{hassani2022curse}, such an analysis would rely on a careful application of random matrix theory, to reduce things to the linear case.

On another front, the analysis in our work applies to any linear model achieving a certain level of standard accuracy. In particular, our analysis is agnostic to whether the model is learned from data or not, and on how it is learned (gradient-descent, adversarial training, etc.). Though we have not studied these statistical regimes, our techniques can be adapted to this scenario to understand possible additional accuracy / robustness tradeoffs due to over-fitting, for example.










\bibliography{literature}
\bibliographystyle{apalike}



\clearpage

\addcontentsline{toc}{section}{Appendix} 
\parttoc 

\appendix
\section{Detailed Overview of Related Works}
\label{sec:related}
\paragraph{Adversarial Examples From High-Dimensional Geometry.} In the setting of classification, \cite{tsipras18} considers a specific data distribution where good accuracy implies poor robustness.
\citep{goldstein,saeed2018, gilmerspheres18,dohmatob19} show that for high-dimensional data distributions which have concentration property (e.g., multivariate Gaussians, distributions satisfying log-Sobolev inequalities, etc.), an imperfect classifier will admit adversarial examples. \cite{dobriban2020provable} studies tradeoffs in Gaussian mixture classification problems, highlighting the impact of class imbalance. On the other hand, \cite{closerlook2020} observed empirically that natural images are well-separated, and so locally-lipschitz classifiers should not suffer any kind of test error vs robustness tradeoff.

\paragraph{The Impact of Over-Parametrization.} \cite{RuiqiGao2019,lor,bubeck2021universal} show that over-parameterization may be necessary for robust interpolation in the presence of noise. In contrast, our paper considers a structured problem with noiseless signal and infinite training data, where the network width $m$ and the input dimension $d$ tend to infinity proportionately. In this under-complete asymptotic setting, our results show a systematic and precise tradeoff between approximation (test error) and robustness in different learning regimes. Thus, our work nuances the picture presented by previous works by exhibiting a nontrivial interplay between robustness and test error, which persists even in the case of infinite training data where the resulting model isn't affected by label noise. \cite{dohmatob2021fundamental,hassani2022curse} study the tradeoffs between interpolation, predictive performance (test error), and robustness for finite-width over-parameterized networks in kernel regimes with noisy linear target functions. In contrast, we consider structured quadratic target functions and compare different learning settings, including SGD optimization in a non-kernel regime, as well as lazy/linearized models.

\paragraph{Precise Analysis of Robustness in Linear Regression.}
\cite{Xing2021} studied Euclidean-norm attacks with general covariance matrices. They showed that the optimal robust model is a ridge regression whose ridge parameter depends implicitly on the strength of the attacks. \cite{Javanmard2020PreciseTI} studied tradeoffs between ordinary and adversarial risk in linear regression, and computed exact Pareto optimal curves in the case of Euclidean-norm attacks on isotropic features. Their results show a tradeoff between ordinary and adversarial risk for adversarial training. \cite{javanmard2021adversarial} also revisited this tradeoff for latent models and show that this tradeoff is mitigated when the data enjoys a low-dimensional structure. The analysis in \cite{Javanmard2020PreciseTI} is based on \emph{Gordon's Comparison Inequality} \cite{gordon88,Thrampoulidis15,Thrampoulidis18}, which is a very versatile tool in the analysis of regularized estimators but fails to produce analytic results when one deviates from the setting of Euclidean-norm attacks on isotropic features. In contrast, our analysis is based on basic Langrangian duality. It relies on some approximations which turn out to only introduce multiplicative absolute constants in the final result, but are completely harmless for the final analysis and interpretation.

Finally, the study of robustness of gradient-descent in the context of linear regression under general-norms attacks and feature covariance matrices has been initiated in \cite{MeyerAndDohmatob2023} which gave sufficient conditions for the generative model $w_0$ (and its estimators like gradient descent, ridge regression, etc.) to be robust. However, the the question of tradeoffs was not considered.
\section{Additional Experimental Results}
\subsection{Experiment 2 (Extended)}
We provide further empirical confirmation for Theorem \ref{thm:polydecay}. Figures \ref{fig:poly-decay-eps-suppmat} and \ref{fig:poly-Eopt-suppmat} are complementary to Figures \ref{fig:poly-decay-eps} and \ref{fig:poly-Eopt} respectively in the main text. They show results for \textbf{Experiment 2} (refer to Section \ref{sec:experiments}) for other values of the exponents $\beta$ and $\delta$.

\begin{figure}[!h]
    \centering
    \includegraphics[width=1\textwidth]{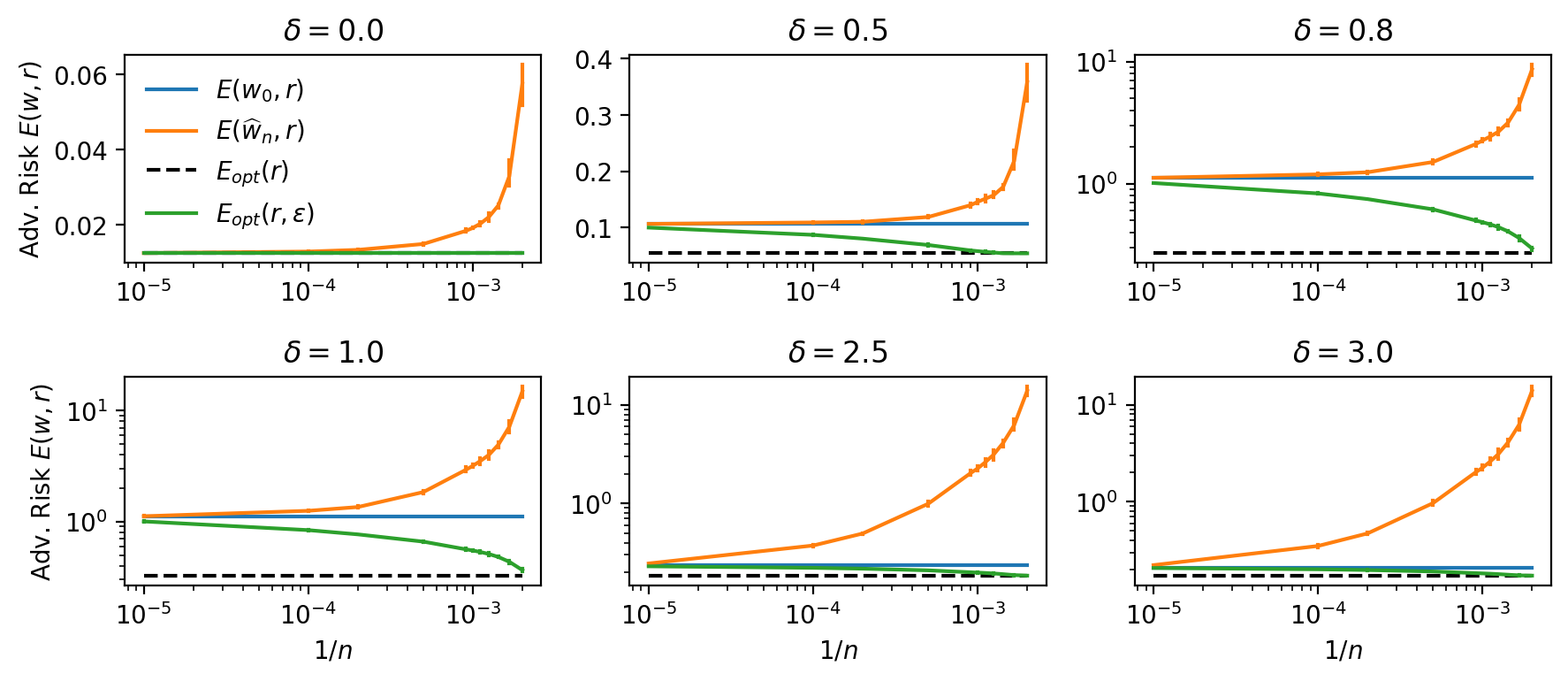}
\vspace{-.6cm}
    \caption{(\textbf{Experiment 2, extended}) Empirical verification of Theorem \ref{thm:polydecay}. Here $\beta=1.4$ and $d=10^4$. As in Figure \ref{fig:poly-decay-eps}, notice the conformity of the results with the theorem, namely: if the the model $\widehat w_n$ is accurate (small $\epsilon$), then it is robust (compared to the optimal achievable adversarial risk $E_{opt}(r)$) for $\delta \in (1,\infty)$, but non-robust for $\delta \in [0,1)$.
    }
    \label{fig:poly-decay-eps-suppmat}
\end{figure}

\begin{figure}[!hbt]
    \centering
    \includegraphics[width=1\textwidth]{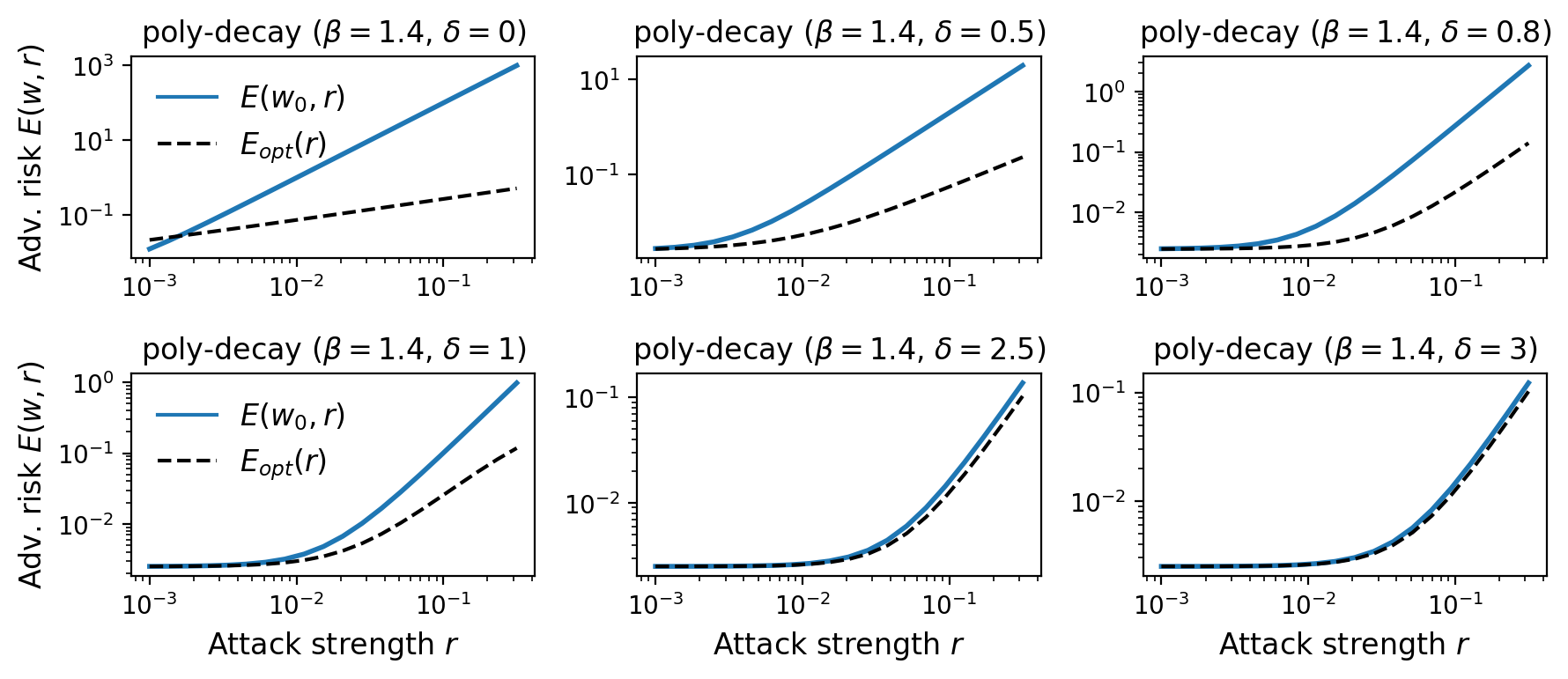}
    \caption{\textbf{(Experiment 2, extended)} Empirical validation of Theorem \ref{thm:polydecay} for the case when $\epsilon=0$. Here $\beta=1.4$. As in Figure \ref{fig:poly-Eopt}, notice the conformity of the results with the theorem, namely: the generative model $w_0$ is robust (compared to the optimal achievable adversarial risk $E_{opt}(r)$) for $\delta \in (1,\infty)$ but non-robust for $\delta \in [0,1)$.}
    \label{fig:poly-Eopt-suppmat}
\end{figure}
\section{Proof of Theorem \ref{thm:freelunch} (Main Result) and Theorem \ref{thm:Eopt}}
\freelunch*

\begin{proof}
Write $\overline E_{opt}(r,\epsilon) := \inf_{w \in \mathcal W_\epsilon}\overline E(w,r)$, where we recall that
$$
\mathcal W_\epsilon := \{w \in \mathbb R^d \mid \Delta(w) \le \epsilon^2\} = \{w \in \mathbb R^d \mid \|w-w_0\|_\Sigma^2 \le \epsilon^2\|w_0\|_\Sigma^2\}.
$$
Thus, if $\epsilon \ge \epsilon_{FL}(r):=\sqrt{\Delta(w^{prox}(r^2))} = \sqrt{G(r^2)}/\|w_0\|_\Sigma$, then $w^{prox}(r^2) \in \mathcal W_\epsilon$, and so we deduce from Lemma \ref{lm:proxy} that
\begin{eqnarray}
E_{opt}(r,\epsilon) \asymp \overline E_{opt}(r,\epsilon) = \overline E(w^{prox}(r^2),r) \asymp E(w^{prox}(r^2),r).
\end{eqnarray}

Henceforth, suppose $0 \le \epsilon \le \epsilon_{FL}(r)$. First observe that, for every $\lambda \in [0,r^2]$,
\begin{eqnarray}
\label{eq:trick}
    F(r,\lambda) = \inf_{\|w-w_0\|_\Sigma^2 \le G(\lambda)} \overline E(w,r).
\end{eqnarray}
Indeed, let $w \in \mathbb R^d$ such that $\|w-w_0\|_\Sigma^2 \le G(\lambda) := \|w^{prox}(\lambda)-w_0\|_\Sigma^2$. Let $t \ge 0$ such that $\lambda = r^2/(1+t)$, and set $L_t(w,r) := \overline E(w,r) + t\|w-w_0\|_\Sigma^2$. By definition of $w^{prox}(\lambda)$ in \eqref{eq:wt}, one has
\begin{eqnarray*}
\begin{split}
   F(r,\lambda) + t G(\lambda) &= \overline E(w^{prox}(\lambda),r) + t G(\lambda) = L_t(w^{prox}(\lambda),r)\\
   &= \|w^{prox}(\lambda)-w_0\|_\Sigma^2 + r^2\|w^{prox}(\lambda)\|_\star^2 + t\|w^{prox}(\lambda)-w_0\|_\Sigma^2\\
   &= (1+t)(\|w^{prox}(\lambda)-w_0\|_\Sigma^2 + \lambda \|w^{prox}(\lambda)\|_\star^2)\\
   &\le (1+t)(\|w-w_0\|_\Sigma^2 + \lambda \|w\|_\star^2)\text{ by definition of }w^{prox}(\lambda)\\
   &= L_t(w,r) = \overline E(w,r) + t\|w-w_0\|_\Sigma^2,
    \end{split}
\end{eqnarray*}
and it follows that $\overline E(w^{prox}(\lambda),r) = F(r,\lambda) \le \overline E(w,r) + t (\|w-w_0\|_\Sigma^2 - G(\lambda)) \le \overline E(w,r)$.

Now, equipped with \eqref{eq:trick} and the definition of $\overline E_{opt}(r,\epsilon)$, observe that
\begin{itemize}
    \item[--] If $G(\lambda) \le \epsilon^2\|w_0\|_\Sigma^2$, then $\overline E_{opt}(r,\epsilon) \ge F(r,\lambda)$.
    \item[--] Analogously, if $G(\lambda) \ge \epsilon^2\|w_0\|_\Sigma^2$, then $\overline E_{opt}(r,\epsilon) \le F(r,\lambda)$.
\end{itemize}
On the other hand, Lemma \ref{lm:decreasing} tells us that the equation $G(\lambda) = \epsilon^2\|w_0\|_\Sigma^2$ has a unique solution $\lambda_{opt}(r,\epsilon)$ in $[0,r^2]$, and we deduce from Lemma \ref{lm:proxy} that
$$
E_{opt}(r,\epsilon) \asymp \overline E_{opt}(r,\epsilon) = \overline E(w^{prox}(\lambda_{opt}(r,\epsilon)),r) \asymp E(w^{prox}(\lambda_{opt}(r,\epsilon)),r),
$$
which completes the proof.
\end{proof}

\subsection{Proof of Theorem \ref{thm:Eopt}}
\Eopt*
\begin{proof}
  Indeed, by Lemma \ref{lm:bigeps} we know that $E_{opt}(r) = E_{opt}(r,1)$. Also, by Lemma \ref{lm:implicit}, we know that $\lambda_{opt}(r,1) = r^2$. Combining these with part (A) of Theorem \ref{thm:freelunch} then gives the result.
\end{proof}
\subsection{A Corollary: Isotropic Features}
As an important corollary to Theorem \ref{thm:freelunch} (not stated in the main manuscript), consider the case of isotropic features considered in \cite{Javanmard2020PreciseTI}, where $\Sigma=I_d$.
\begin{restatable}{thm}{isotropic}
Consider the isotropic setting where $\Sigma=I_d$. For Euclidean-norm attack of strength $r \ge 0$, it holds for any tolerance level $\epsilon \in [0,1]$ that

\begin{itemize}
\item[(i)] (\textbf{Free Lunch Threshold}) $\epsilon_{FL}(r) = r^2/(1+r^2) \in [0,1)$.

\item[(ii)] (\textbf{Free Lunch}) If $\epsilon \ge \epsilon_{FL}(r)$, then the optimal regularization is $\lambda_{opt}(r,\epsilon) = r^2$, and we have
\begin{eqnarray}
E_{opt}(r,\epsilon) \asymp E_{opt}(r) \asymp  \sigma^2 + \|w_0\|_2^2\min(r^2,1).
\end{eqnarray}

\item[(ii)] (\textbf{Accuracy / Robustness Tradeoff}) If $\epsilon < \epsilon_{FL}(r)$, then the optimal regularization parameter is given by 
  $\lambda_{opt}(r,\epsilon) = \epsilon/(1-\epsilon)$, and we have
\begin{eqnarray}
    E_{opt}(r,\epsilon) \asymp \sigma^2 + \|w_0\|_2^2(\epsilon^2 + (1-\epsilon)^2 \min(r^2,1)).
\end{eqnarray}
\end{itemize}
\label{thm:isotropic}
\end{restatable}

\section{Structure of Optima}
\label{sec:structure}
In this section, we explore the structure of the curve $\lambda \mapsto w^{prox}(\lambda)$ given in \eqref{eq:wt} for different choices of the attacker's norm $\|\cdot\|$.

\subsection{The Case of Mahalanobis-Norm Attacks}
Suppose the attacker's norm $\|\cdot\|$ is the Mahalanobis norm $\|\cdot\|_B$ induced by a positive-definite $d \times d$ matrix $B$. Then, for any $\lambda \ge 0$,  $w^{prox}(\lambda)$ minimizes $\|w-w_0\|_\Sigma^2 + \lambda\|w\|_\star^2 = \|w-w_0\|_\Sigma^2 + \lambda\|w\|_{B^{-1}}^2$, which gives the closed-form solution
\begin{eqnarray}
    w^{prox}(\lambda) = (\Sigma + \lambda B^{-1})^{-1} \Sigma w_0 = (B\Sigma + \lambda I_d)^{-1}B\Sigma w_0.
    \label{eq:wB}
\end{eqnarray}
Also, note that
\begin{eqnarray}
\begin{split}
G(\lambda) &:= \|w^{prox}(\lambda)-w_0\|_\Sigma^2 = \|((\Sigma + \lambda B^{-1})^{-1} \Sigma - I_d)w_0\|_\Sigma^2\\
&= \lambda^2\|(B\Sigma + \lambda I_d)^{-1}w_0\|_\Sigma^2\\
F(r,\lambda) - G(\lambda) &= r^2\|w^{prox}(\lambda)\|^2_2 = r^2\|(\Sigma + \lambda B^{-1})^{-1} \Sigma w_0\|_2^2\\
&= r^2\|(B\Sigma + \lambda I_d)^{-1} B\Sigma w_0\|_2^2.
\end{split}
\end{eqnarray}

\paragraph{Link with \cite{MeyerAndDohmatob2023}.}
Note that \eqref{eq:wB} recovers the structure established in \cite{MeyerAndDohmatob2023}, where $1/\lambda$ should be thought of as the time parameter in the population-wise \emph{adapted} (i.e. pre-conditioned) gradient-flow (GD+) proposed in that work, with the choise $M=B^{1/2}$. We deduce the following:
\begin{itemize}
\item GD+ started from zero and run for time $O(1/r^2)$ achieves the optimal adversarial risk $E_{opt}(r)$ (up to within multiplicative absolute constants). This follows from Theorem \ref{thm:Eopt} and the preceding argument.
\item More generally, for a tolerance parameter $\epsilon \in [0,1]$, GD+ started from zero and run for time $O(1/\lambda_{opt}(r,\epsilon))$ achieves the optimal adversarial risk, where $\lambda_{opt}(r,\epsilon) \in [0,r^2]$ is as given in Lemma \ref{lm:implicit}.
\end{itemize}

\paragraph{Link with \cite{Xing2021}.} In particular, for Euclidean-norm attacks corresponding to $B=I_d$, \eqref{eq:wB} reduces to
\begin{eqnarray}
w^{prox}(\lambda) = (\Sigma + \lambda I_d)^{-1} \Sigma w_0,
\end{eqnarray}
which recovers the structure established in \cite{Xing2021}.

\subsection{The Case of $\ell_p$-norm Attacks on Diagonal Feature Covariance Matrix}
Suppose the feature covariance matrix is $\Sigma=\mathrm{diag}(\lambda_1,\ldots,\lambda_d)$ and the attacker's norm is $\|\cdot\|_p$, for some $p \in [1,\infty]$. Let $q \in [1,\infty]$ be the harmonic conjugate of $p$. By \eqref{eq:wt}, $w^{prox}(\lambda)$ is the minimizer of $\|w-w_0\|_\Sigma^2 + \lambda \|w\|_q^2$ over $w \in \mathbb R^d$. If $R_q(\lambda):= \|w^{prox}(\lambda)\|_q$, then by first order optimality conditions, we have $\Sigma w_0 \in  \Sigma w + \lambda R(\lambda) \partial \|\cdot\|_q(w)$, i.e
\begin{eqnarray}
    w^{prox}(\lambda) = (I + R_q(\lambda)\lambda \partial \|\cdot\|_q)^{-1}(\Sigma w_0) = \mbox{prox}_{t_q(\lambda) \|\cdot\|_q}(\Sigma w_0),
\end{eqnarray}
where $t_q(\lambda) := R_q(\lambda)\lambda$.

\paragraph{The Case of $\ell_\infty$-Norm Attacks.}
In the special case where $p=\infty$, we have $q=1$, giving
\begin{eqnarray}
\label{eq:st}
    w^{prox}(\lambda)_k = \mathrm{ST}(\mu_k; t_1(\lambda)) = \begin{cases}
        \mu_k + t_1(\lambda),&\mbox{ if }\mu_k < -t_1(\lambda),\\
        0,&\mbox{ if }|\mu_k| \le t_1(\lambda),\\
         \mu_k - t_1(\lambda),&\mbox{ if }\mu_k > t_1(\lambda),
    \end{cases}
\end{eqnarray}
where $\mu_k := \lambda_k \cdot (w_0)_k$ for all $k \in [d]$ (i.e $\mu = \Sigma w_0$) and $\mathrm{ST}$ is the well-known \emph{soft-thresholding (ST)} operator. Thus, if $t_1(\lambda) \ge \|\mu\|_\infty$, then $w^{prox}(\lambda) = 0$. This means that we can always restrict our search of the optimal threshold $t_1(\lambda)$ to a compact interval,
\begin{eqnarray}
t_1(\lambda) \in [0,\|\mu\|_\infty].
\end{eqnarray}
The structure of the optimal \eqref{eq:st} is instructive: components $(w_0)_k$ of $w_0$ corresponding to to features with small values of $|\mu_k|$ are zeroed-out.


\section{Well-Conditioned Problems}
\subsection{Estimating $E_{opt}(r)$}
We now give a complete analysis of $E_{opt}(r)$ for the case of so-called "well-conditioned" problems (formally defined later).
To develop an intuition, first consider the simple case of Euclidean-norm attacks on isotropic features (i.e. $\Sigma=I_d$). In this case, Lemma \ref{lm:FGanalytic} tells us that
\begin{eqnarray}
G(\lambda) = \lambda^2\|w_0\|_2^2/(1+\lambda)^2,\, F(r,\lambda) = G(\lambda) + r^2\|w_0\|_2^2/(1+\lambda)^2.
\end{eqnarray}
Theorem \ref{thm:Eopt} then predicts that
\begin{eqnarray}
\begin{split}
    E_{opt}(r) &\asymp \sigma^2 + F(r,r^2) = \sigma^2 + \frac{r^4}{(1+r^2)^2}\|w_0\|_2^2 + \frac{r^2}{(1+r^2)^2}\|w_0\|_2^2\\
    &= \sigma^2 + \frac{r^2}{1+r^2}\|w_0\|_2^2 \asymp \sigma^2 + \|w_0\|_2^2 \min(r^2,1) \asymp \min(E(w_0,r),E(0,r)).
\end{split}
\end{eqnarray}
Thus, for $r \le 1$, the generative model $w_0$ attains the optimal adversarial risk $E_{opt}(r)$ (upto within multiplicative constant); for $r \ge 1$, the optimal adversarial risk is attained by the null model $w=0$. This recovers a result of \cite{MeyerAndDohmatob2023}. 

We now consider the situation of general norms and covariance matrices. Define $r_0, r_1,\eta > 0$ by
\begin{eqnarray}
\begin{split}
r_0  := \|w_0\|_\Sigma/\|w_0\|_\star,\,
r_1  := \|\Sigma w_0\|/\|w_0\|_\Sigma,\,
\eta_0 :=  r_1/r_0.
\end{split}
\label{eq:rOeta0}
\end{eqnarray}
Note that $\eta_0 \ge 1$ by Cauchy-Schwarz inequality. This scalar should be thought of as a kind of \emph{condition number} for $\Sigma$ w.r.t the attacker's norm $\|\cdot\|$. In particular, when this norm is Euclidean, then $\eta_0$ is upper-bounded by the usual linear-algebraic condition number of $\Sigma$.
In general, $r_1 = r_0\eta_0 \ge r_0$, with equality when $\eta_0=1$, which is the case when for example one considers Euclidean-norm attacks on isotropic features, i.e. $\Sigma = I_d$.

\begin{restatable}{df}{}
By \emph{well-conditioned} problems, we mean scenarios where $\eta_0 = O(1)$. 
\end{restatable}

\begin{restatable}{thm}{}
\label{thm:well-conditioned}
For an attack strength $r \ge 0$ w.r.t a general norm $\|\cdot\|$, it holds that 
\begin{eqnarray}
    \sigma^2 + \|w_0\|_\Sigma^2 \min(r/r_1,1)^2 \lesssim E_{opt}(r) \lesssim \sigma^2 + \|w_0\|_\Sigma^2 \min(r/r_0,1)^2.
\end{eqnarray}
In particular, for well-conditioned problems (i.e. $\eta_0 = O(1)$), it holds that
\begin{eqnarray}
E_{opt}(r) \asymp \sigma^2 + \|w_0\|_\Sigma^2\min(r/r_0,1)^2 \asymp \min(E(w_0,r),E(0,r)).
\end{eqnarray}
\end{restatable}
Thus, for well-conditioned problems, the generative model $w_0$ is optimally robust (up to an absolute multiplicative constant) for small values of $r$ (i.e. $r \le r_0$), while for large values of $r$ ($r \ge r_0$), the null model $w=0$ is optimally robust (up to an absolute multiplicative constant).

\subsection{Estimating $E_{opt}(r,\epsilon)$}
We now generalize the results of the previous subsection and establish some results which are complementary to the results in Sections \ref{sec:lagrangian}. These use different techniques but arrive at qualitatively and quantitatively similar results in certain settings.

Define an auxiliary function $H:\mathbb R_+ \times \mathbb R_+ \to \mathbb R_+$ by
\begin{align}
H(r,\epsilon) &:= \begin{cases}
    r,&\mbox{ if }0 \le r < 1,\\
    \delta + (1-\delta)r,&\mbox{ else,}
    \end{cases},\label{eq:H}
\end{align}
where $\delta = \delta(\epsilon) := \min(1,\epsilon)$. 
This is the same function which appears in Theorem \ref{thm:isotropic-eps-different-p} (see Table \ref{tab:isotropic-eps}). Note that $r=H(r,0) = H(r,\epsilon) \ge H(r,1) = \min(r,1)$, for all $r \ge 0$ and $\epsilon \in [0,1]$.
The following which holds for any attacker norm, is one of our main contributions.
\begin{restatable}{thm}{lousy}
\label{thm:lousy}
For any $r \ge 0$ and $\epsilon \in [0,1]$, the following bounds hold
\begin{eqnarray}
\label{eq:lousybounds}
    \sigma^2 
 + \|w_0\|_\Sigma^2H(r/r_1,\epsilon)^2 \lesssim E_{opt}(r,\epsilon) \lesssim \sigma^2 
 + \|w_0\|_\Sigma^2H(r/r_0,\epsilon)^2.
\end{eqnarray}
In particular, if $r_0 \asymp r_1$, then $E_{opt}(r,\epsilon) \asymp \sigma^2 
 + \|w_0\|_\Sigma^2H(r/r_0,\epsilon)^2$.
\end{restatable}
Note that when $\epsilon=1$, we have $H(r,\epsilon) = H(r,1) = \min(r,1)$ for any $r \ge 0$, and the above result recovers Theorem \ref{thm:well-conditioned} as a special case.

\subsection{The Case of Euclidean-Norm Attacks}
In the special case of Euclidean-norm attacks, one computes
$$
1 \le \eta_0 = \frac{r_0}{r_1} = \frac{\|\Sigma w_0\|_2\|w_0\|_2}{\|w_0\|_\Sigma^2} = \frac{\|\Sigma^{1/2}\Sigma^{1/2} w_0\|_2}{\|\Sigma^{1/2} w_0\|_2}\frac{\|w_0\|_2}{\|\Sigma^{1/2} w_0\|_2} \le \sqrt{\kappa(\Sigma)},
$$
where $\kappa(\Sigma)$ is the ordinary condition number of the covariance matrix $\Sigma$.

\paragraph{A Well-Conditioned Example.} In particular, in the case of isotropic features where $\Sigma = I_d$, we have $r_0=r_1=\|\Sigma^{1/2}\|_{op} = 1$ and $\eta_0 = 1$, and Theorem \ref{thm:lousy} then gives
\begin{eqnarray}
\begin{split}
    E_{opt}(r,\epsilon) &\asymp \sigma^2 + \|w_0\|_\Sigma^2 \cdot H(r/r_0,\epsilon)^2 = \sigma^2 + (\epsilon^2 + (1-\epsilon)^2r^2)\|w_0\|_\Sigma^2,
    \end{split}
    \label{eq:well-conditioned}
\end{eqnarray}
for all $r \ge r_0$.
This is exactly the result obtained in Theorem \ref{thm:isotropic}.

\paragraph{A Case of Failure for Theorem \ref{thm:lousy}.} Note that even in the Euclidean case, Theorem \ref{thm:lousy} might become vacuous when the "condition number" $\eta_0$ is too large. This is for example, the case of polynomial decaying eigenvalues of the covariance matrix $\Sigma$, considered in Section \ref{subsec:polydecay}. Indeed, that example with $\delta \in (1,\infty)$, one easily computes $r_0 = \|w_0\|_\Sigma/\|w_0\|_2  \to 0$ in the limit $\delta \to 1^+$, and $r_1 = \|\Sigma w_0\|_2/\|w_0\|_\Sigma = \Theta(1)$, and the lower-bound in \eqref{eq:lousybounds} becomes vacuous. However, the results of Section \ref{subsec:polydecay} (Theorem \ref{thm:polydecay}) remain valid even in this ill-conditioned limit.



\subsection{Sketch of Proof of Theorem \ref{thm:lousy}}
\label{subsec:lousyproof}
We now outline the key ideas underlying  the proof of Theorem \ref{thm:lousy}, 
split into various steps. The details are provided in Section \ref{sec:lousyproofdetails}.




\paragraph{Step 1: Proxy for Adversarial Risk.}
From Lemma \ref{lm:proxy}, we know that $E(w,r) \asymp \widetilde E(w,r)$, and so
\begin{eqnarray}
E_{opt}(r,\epsilon) \asymp \sigma^2 + K_{opt}(r,\epsilon)^2,
\end{eqnarray}
where $\widetilde E(w,r) := \|w-w_0\|_\Sigma^2 + r^2\|w\|_\star^2$ and $K_{opt}(r,\epsilon) := \inf_{w \in \mathcal W_\epsilon} K(w,r)$.

\paragraph{Step 2: Restricting the Search to a Chord.}
Computing $K_{opt}(r,\epsilon)$, even though conceivably easier than $E_{opt}(r,\epsilon)$, is still difficult. Instead, we will restrict the optimization to a line / chord in $\mathcal W_\epsilon$, parallel to the generative model $w_0$. It will turn out that up to within multiplicative constants, this strategy gives the correct value of $K_{opt}(r,\epsilon)$ as a function of all relevant problem parameters. To this end, let $K_{shrink}(r,\epsilon)$ be the optimal adversarial risk achieved by a linear model which is co-linear with the generative model $w_0$, i.e
\begin{eqnarray}
K_{shrink}(r,\epsilon) := \inf_{w \in \mathcal W_\epsilon \cap \langle w_0\rangle}K(w,r),
\end{eqnarray}
where $\langle w_0\rangle := \{t w_0 \mid t \in \mathbb R\}$ is the one-dimensional subspace of $\mathbb R^d$ spanned by the generative model $w_0$.

\begin{restatable}{prop}{keyprop}
For any $r,\epsilon \ge 0$, the following bounds hold
\label{prop:keyprop}
\begin{align}
 K_{shrink}(r\eta_0,\epsilon) \le K_{opt}(r,\epsilon) \le K_{shrink}(r,\epsilon).
\end{align}
\end{restatable}
This auxiliary result, which is the main component of the proof of Theorem \ref{thm:lousy}, is proved in Section \ref{sec:lousyproofdetails}.

\paragraph{Step 3: Computing the Value of $K_{shrink}(r,\epsilon)$.} To complete the proof of Theorem \ref{thm:lousy}, it remains to show that 
the proxy $K_{shrink}(r,\epsilon)$ equals $H(r/r_0,\epsilon)^2$ up to within multiplicative absolute constants. The proof of Theorem \ref{thm:lousy} would then follow upon plugging such estimates into the bounds given in Proposition \ref{prop:keyprop}.
\begin{restatable}{prop}{shrinky}
For any $r,\epsilon \ge 0$, it holds that 
\begin{eqnarray}
\dfrac{K_{shrink}(r,\epsilon)}{\|w_0\|_\Sigma^2 }  \asymp \frac{K(t_{opt}w_0,r)}{\|w_0\|_\Sigma^2} \asymp H(\frac{r}{r_0},\epsilon)^2,
\end{eqnarray}
where where the function $H$ is as defined in \eqref{eq:H}, and $r_0$ is the scalar defined in \eqref{eq:rOeta0}, and $t_{opt}=t_{opt}(r/r_0,\epsilon) \in [0,1]$ is the optimal
and where the function $T$ is as defined in \eqref{eq:T}.
\label{prop:shrinky}
\end{restatable}
Comparing Propositions \ref{prop:keyprop} and \ref{prop:shrinky}, it becomes clear how the function $H$ and $T$ enter the bounds in Theorem \ref{thm:lousy}.

\section{Details of the Proof of Theorem \ref{thm:lousy}}
\label{sec:lousyproofdetails}
\lousy*
The proof was sketched in Section \ref{subsec:lousyproof} with the help of auxiliary propositions, namely Proposition \ref{prop:keyprop} and \ref{prop:shrinky}. Here we just need to provide the proofs for these propositions.

\subsection{Proof of Proposition \ref{prop:shrinky}}
\shrinky*
\begin{proof}
Define an auxiliary function $t_{opt}:\mathbb R_+^2 \to \mathbb R_+$ by
\begin{eqnarray}
    t_{opt}(r,\epsilon) := \begin{cases}1,&\mbox{ if } 0 \le r < 1,\\
   1-\delta,&\mbox{ if }r \ge 1,
     \end{cases}
     \label{eq:T}
\end{eqnarray}
with $\delta=\delta(\epsilon) := \min(1,\epsilon)$ as before.
Also define $K_{shrink}(r,\epsilon)$ by
\begin{eqnarray}
 K_{shrink}(r,\epsilon) :=  \inf_{w \in \mathcal W_\epsilon \cap \langle w_0\rangle}K(w,r).
\end{eqnarray}
By definition of the set $\mathcal W_\epsilon$, note that $w \in \mathcal W_\epsilon \cap \langle w_0\rangle $ iff $w=tw_0$ for some $t \in \mathbb R$ such that $|t-1| \le \epsilon$. Thus, noting that $K(w,r) := \|w-w_0\|_\Sigma^2 + r^2\|w\|_\star^2 \asymp (\|w-w_0\|_\Sigma + r\|w\|_\star)^2$, one computes
\begin{equation*}
    \begin{split}
    \sqrt{K_{shrink}(r,\epsilon)} &:= \inf_{w \in \mathcal W_\epsilon \cap \langle w_0\rangle } \sqrt{K(w,r)}\\
    &\asymp \inf_{|t-1| \le \epsilon} 
 \|w_0\|_\Sigma|t-1| + r\|w_0\|_\star |t|\\
    & = \|w_0\|_\Sigma \cdot \inf_{|t-1| \le \epsilon} |t-1| + \frac{r\|w_0\|_\star}{\|w_0\|_\Sigma}|t|\\
    & = \|w_0\|_\Sigma \cdot \inf_{|t-1| \le \epsilon} |t-1| + \frac{r}{r_0}|t|\\
    & = \|w_0\|_\Sigma \cdot H(r/r_0,\epsilon),
    \end{split}
\end{equation*}
where the last line is thanks to Lemma \ref{lm:thorny}.
\end{proof}

\subsection{Proof of Proposition \ref{prop:keyprop}}
\keyprop*
\begin{proof}
Let $C = \langle w_0\rangle:=\{t w_0 \mid t \in \mathbb R\} \subseteq \mathbb R^d$ be the one-dimensional subspace spanned by the generative model $w_0$, and let $P_{C,\Sigma}:\mathbb R^d \to C$ be the projection operator onto $C$, w.r.t the the Mahalanobis norm $\|\cdot\|_\Sigma$. Then, by non-expansiveness of $P_{C,\Sigma}$ (see \cite{BauschkeCombettesConvexAnalysis2011}, for example), one has for any $w \in \mathbb R^d$,
$$
\|P_{C,\Sigma}(w)-w_0\|_\Sigma = \|P_{C,\Sigma}(w)-P_{C,\Sigma}(w_0)\|_\Sigma \le \|w-w_0\|_\Sigma. 
$$
Now, for any $w \in \mathbb R^d$, we have $P_{C,\Sigma}(w) = t w_0$, where $t \in \mathbb R$ minimizes $f(t):=\|w-t w_0\|_\Sigma^2$. Now, $f'(t) = 2w_0^\top \Sigma (t w_0 - w) = 2(\|w_0\|_\Sigma ^2t - w^\top \Sigma w_0)$. Thus, the optimal $t$ is $w^\top \Sigma w_0 / \|w_0\|_\Sigma^2$, and so
\begin{eqnarray}
P_{C,\Sigma}(w) = \frac{w^\top\Sigma w_0}{\|w_0\|_\Sigma^2}w_0.
\end{eqnarray}

Let us now bound the operator norm of $P_{C,\Sigma}$ w.r.t the dual norm $\|\cdot\|_\star$. For any $w \in \mathbb R^d$, one has
\begin{eqnarray*}
\begin{split}
\frac{\|P_{C,\Sigma}(w)\|_\star}{\|w\|_\star} = \frac{\|(w^\top \Sigma w_0)w_0\|_\star}{\|w\|_\star\|w_0\|_\Sigma^2} &= \frac{|w^\top \Sigma w_0|}{\|w\|_\star\|w_0\|_\Sigma}\frac{\|w_0\|_\star}{\|w_0\|_\Sigma}\le \frac{\|\Sigma w_0\|}{\|w_0\|_\Sigma}\frac{\|w_0\|_\star}{\|w_0\|_\Sigma} =:\eta_0,
\end{split}
\end{eqnarray*}
where the second line is an application of the Cauchy-Schwarz inequality. We deduce that
\begin{eqnarray*}
\begin{split}
\sqrt{K(P_{C,\Sigma}(w),r)}
&\asymp \|P_{C,\Sigma}(w)-w_0\|_\Sigma + r\|P_{C,\Sigma}(w)\|_\star\\
&\le \|w-w_0\|_\Sigma + r\eta_0 \|w\|_\star \asymp \sqrt{K(w,r\eta_0)}.
\end{split}
\end{eqnarray*}
Thus, $K(w,r) \gtrsim K(P_{C,\Sigma}(w), r/\eta_0)$. On the other hand, if $w \in \mathcal W_\epsilon$, then the non-expansiveness of $P_{C,\Sigma}$ (again!) gives
$$
\|P_{C,\Sigma}(w) - w_0\|_\Sigma = \|P_{C,\Sigma}(w) - P_{C,\Sigma}(w_0)\|_\Sigma \le \|w-w_0\|_\Sigma \le \epsilon\|w_0\|_\Sigma,
$$
that is, $P_{C,\Sigma}(w) \in \mathcal W_\epsilon$.
Putting things together yields: for any $w \in \mathcal W_\epsilon$, there exists $z \in \mathcal W_\epsilon \cap C$ such that $K(z,r/\eta_0) \le K(w,r)$. Therefore, 
\begin{eqnarray*}
\begin{split}
K_{opt}(r,\epsilon)  &:= \inf_{w \in \mathcal W_\epsilon} K(w,r)
\gtrsim \inf_{z \in \mathcal W_\epsilon \cap C} K(z,r/\eta_0) =: K_{shrink}(r/\eta_0,\epsilon).
\end{split}
\end{eqnarray*}
This establishes the lower-bound Proposition \ref{prop:keyprop}.

As for the upper-bound, one computes
\begin{eqnarray*}
\begin{split}
K_{opt}(r,\epsilon) &:= \inf_{w \in \mathcal W_\epsilon} K(w,r) \le \inf_{w \in \mathcal W_\epsilon \cap C} K(w,r) =: K_{shrink}(r,\epsilon),
\end{split}
\end{eqnarray*}
as claimed.
\end{proof}

\section{Technical Proofs}
\subsection{Proof of Lemma \ref{lm:analytic}: Analytic Formula for Adversarial Risk}
\analytic*
For the proof, we will need the following auxiliary lemma.
\begin{restatable}{lm}{}
For any $x,w \in \mathbb R^d$, $r \ge 0$, and $y \in \mathbb R$, the following identity holds
\begin{eqnarray}
    \sup_{\|\delta\| \le r}|(x+\delta)^\top w - y| = |x^\top w - y| + r\|w\|_\star.
\end{eqnarray}
\label{lm:wellknown}
\end{restatable}
\begin{proof}
Note that $h(x,y,\delta)/2 = \eta(x,y)^2/2 + g(x,y,\delta)/2$, where $g(x,y,\delta) := w(\delta)^2-2\eta(x,y)w(\delta)$, and $\eta(x,y) := w(x) - y$, and $w(x) := x^\top w$. Now, because the real function $z \mapsto z^2/2$ is its own Fenchel-Legendre conjugate, we can "dualize" our problem as follows
\begin{eqnarray*}
\begin{split}
\sup_{\|\delta\| \le r}g(x,y,\delta)/2 &= \sup_{\|\delta\|_\star \le r}-\eta(x,y)w(\delta) + \sup_{z \in \mathbb R} zw(\delta) - z^2/2\\
&=\sup_{z \in \mathbb R}-z^2/2 + \sup_{\|\delta\| \le r}(z-\eta(x,y))w(\delta)\\
&= \sup_{z \in \mathbb R}r\|w\|_\star|z-\eta(x,y)| - z^2/2\\
&= \sup_{s \in \{\pm 1\}}\sup_{z \in \mathbb R}rs(z-\eta(x,y))-z^2/2\\
&= \sup_{s \in \{\pm 1\}}-r\|w\|_\star s\eta(x,y)+\sup_{z \in \mathbb R}r\|w\|_\star sz-z^2/2\\
&= \sup_{s \in \{\pm 1\}}-r\|w\|_\star s\eta(x,y) + r^2\|w\|_\star^2 /2\\
&= r\|w\|_\star |\eta(x,y)|+r^2\|w\|^2_\star /2.
\end{split}
\end{eqnarray*}
We deduce that
\begin{align*}
\sup_{\|\delta\| \le r}h(x,y,\delta)/2 &= \eta(x,y)^2/2+r\|w\|_\star |\eta(x,y)|+r^2\|w\|_\star^2 /2 = (|\eta(x,y)|+r\|w\|_\star)^2/2,
\end{align*}
from which the result follows.
\end{proof}

\begin{proof}[Proof of Lemma \ref{lm:analytic}]
Indeed, thanks to Lemma \ref{lm:wellknown},
one has
\begin{eqnarray}
    E(w,r) := \mathbb E\sup_{\|\delta\|\le r} h(x,y,\delta) = \mathbb E[(\eta(x,y)+r\|w\|_\star)^2],
\end{eqnarray}
where the functions $h$ and $\eta$ are as in the proof of Lemma \ref{lm:wellknown}. The result then follows upon noting that, for $x \sim N(0,\Sigma)$ and $y|x \sim N(x^\top w_0,\sigma^2)$,
\begin{eqnarray*}
\begin{split}
\mathbb E[\eta(x,y)^2] &= \mathbb E[(x^\top w - y)^2] = E(w) = \|w-w_0\|_\Sigma^2 + \sigma^2,\\
\mathbb E|\eta(x,y)| &= \mathbb E_x|x^\top w - y| = c_0\sqrt{\|w-w_0\|_\Sigma^2+\sigma^2} = c_0\sqrt{E(w)}, 
\end{split}
\end{eqnarray*}
where $c_0 := \sqrt{2/\pi}$ as in the lemma.
\end{proof}

\subsection{Proof of Lemma \ref{lm:proxy}}
\proxy*
We will need the following elementary lemma.
\begin{restatable}{lm}{}
\label{lm:abc}
For any $a,b,c \ge 0$ with $c \le 1$, it holds that
\begin{eqnarray}
\begin{split}
(a+b)^2 &\geq  a^2 + b^2 + 2abc \geq \frac{1+c}{2}(a+b)^2,\\
a^2 + b^2 &\leq a^2 + b^2 + 2abc \leq (1+c)(a^2 + b^2).
\end{split}
\end{eqnarray}
\label{lm:blend}
\end{restatable}
\begin{proof}
Let $h(a,b,c):=a^2+b^2+2abc$.
For the LHS, it suffices to observe that $h(a,b,c) \le h(a,b,1) = (a+b)^2$. For the RHS, WLOG assume that $a \ne 0$, and set $t:=b/a \ge 0$. Observe
$$
1 \ge \dfrac{h(a,b,c)}{(a+b)^2} = \dfrac{1+ t^2 + 2ct}{(1+t)^2},
$$
and the RHS is minimized when $t=1$, because $0 \le c \le 1$ by assumption. We deduce that $h(a,b,c)/(a+b)^2 \geq (1+1+2c)/(1+1)^2 = (1+c)/2$. This proves the first line of inequalities in the lemma.

On the other hand, observe that
\begin{eqnarray}
    1 \le \frac{h(a,b,c)}{a^2 + b^2} = \frac{1+t^2 + 2ct}{1+t^2} = 1 + \frac{2ct}{1+t^2}.
\end{eqnarray}
Clearly, the RHS attains a maximum value of $1+c$ at $t=1$. This proves the second line of inequalities in the lemma.
\end{proof}

We are now ready to proof Lemma \ref{lm:proxy}.
\begin{proof}[Proof of Lemma \ref{lm:proxy}]
From Lemma \ref{lm:abc} above applied with $c=c_0 = \sqrt{2/\pi}$, we deduce that Lemma \ref{lm:proxy} holds with $c_1 = 2/(1+c_0) \approx 1.11$ and $c_2 = 1+c_0 \approx 1.8$.
\end{proof}
\subsection{Proof of Lemma \ref{lm:bigeps}}
\bigeps*
\begin{proof}
Recall that $E_{opt}(r) := \inf_{w \in \mathbb R^d}E(w,r)$ and $E_{opt}(r,\epsilon) := \inf_{w \in \mathcal W_\epsilon}E(w,r)$, where
$$
\mathcal W_\epsilon := \{w \in \mathbb R^d \mid \|w-w_0\|_\Sigma \le \epsilon\|w_0\|_\Sigma\}.
$$
Observe that, if $w \in \mathbb R^d \setminus \mathcal W_1$, then $\|w-w_0\|_\Sigma > \|w_0\|_\Sigma$. We deduce that
$$
E(w,r) \ge E(w) = \|w-w_0\|_\Sigma^2 + \sigma^2 > \|w_0\|_\Sigma^2 + \sigma^2 = E(0) = E(0,r).
$$

On the other hand, if $\epsilon \ge 1$, then $0 \in \mathcal W_1 \subseteq \mathcal W_\epsilon$. Combining with the above inequality gives
\begin{eqnarray*}
    \begin{split}
E_{opt}(r,\epsilon) &= \inf_{w \in \mathcal W_\epsilon}E(w,r) = \min\left(\inf_{w \in \mathcal W_1}E(w,r),\inf_{w \in \mathcal W_\epsilon\setminus \mathcal W_1}E(w,r)\right)\\
&= \inf_{w \in \mathcal W_1}E(w,r) =: E_{opt}(r,1).
    \end{split}
\end{eqnarray*}
and the proof is complete.
\end{proof}

\subsection{Proof of Lemma \ref{lm:implicit}}
\implicit*
\begin{proof}
Indeed, thanks to \cite[Lemma 4]{NonconvexSuvrit} and the definition of $w^{prox}(\lambda)$ in \eqref{eq:wt} and $G(\lambda)$ in \eqref{eq:FG}, the function $G$ is increasing on $[0,r^2]$ with minimal value $G(0) = °$ and maximal value $G(r^2) = \epsilon_{FL}(r) \|w_0\|_\Sigma^2$. Thus, if $0 \le \epsilon \le \epsilon_{FL}(r)$, then $0 \le \epsilon^2 \|w_0\|_\Sigma^2 \le G(r^2)$, and so $\epsilon^2\|w_0\|_\Sigma^2$ is in the range of $G$ over $\lambda \in [0,r^2]$.
\end{proof}

\subsection{On the Auxiliary Function $H$}
\begin{restatable}{rmk}{HTproperties}
The following properties of the function $H$ are easily verified
\begin{itemize}
    \item[(i)] $H(r,\epsilon) \ge \delta + (1-\delta)r$ for all $r \ge 1$ and $\epsilon \ge 0$.
    \item[(ii)] $H(r,\epsilon) = H(r,1) = \min(r,1)$ for all $r \ge 0$ and $\epsilon \ge 1$.
    \item[(iii)] $H(\eta r,\epsilon) \ge \eta H(r,\epsilon)$ for all $r,\epsilon,\eta \ge 0$.
\end{itemize}
\end{restatable}
The functions $H$ and $T$ are linked by the following lemma.
\begin{restatable}{lm}{thorny}
For any $r,\epsilon \ge 0$, we have
\begin{eqnarray}
    \inf_{|t-1| \le \epsilon}k(t) = H(r,\epsilon) = k(t_{opt}(r,\epsilon)),
\end{eqnarray}
where  $k:\mathbb R \to \mathbb R$ is the function defined by $k(t):= |t-1|+ r|t|$. 
\label{lm:thorny}
\end{restatable}
\begin{proof}
First notice that $k(-t) \ge k(t)$ and $|-t-1| = t + 1 \ge |t-1|$ if $t \ge 0$. Thus, WLOG we may assume $t \ge 0$ in the optimization problem.  Now, consider the change of variable $t=t(u):=1-\sqrt{u}$, for $u \in [0,1]$, so that $k(t) = h(u) := \sqrt{u} + r(1-\sqrt{u})$. Note that $|t-1| \le \epsilon$ iff $0 \le u \le \delta^2$, where $\delta=\delta(\epsilon) := \min(1,\epsilon)$.

Now, $2 h'(u) = (1-r) / \sqrt{u}$. We deduce that $h$ is non-decreasing if $r \in [0,1)$ and non-increasing if $r \ge 1$. Thus,
\begin{eqnarray*}
    \begin{split}
\inf_{|t-1| \le \epsilon}k(t) = \inf_{0 \le u \le \delta^2}h(u) &
=
\begin{cases}
h(0) = r,&\mbox{ if }0 \le r < 1,\\
h(\delta^2) =  \delta + r(1-\delta),&\mbox{ if } r \ge 1.
\end{cases}\\
&=: H(r,\epsilon),
    \end{split}
\end{eqnarray*}
as claimed.

To conclude the proof, one manually checks that $k(T(r,\epsilon)) = H(r,\delta) = H(r,\epsilon)$.
\end{proof}

\section{Link with Pareto Fronts}
Let us now link our methods to the idea of Pareto Fronts employed in \cite{Javanmard2020PreciseTI}. For any $r \ge 0$, consider the \emph{Pareto Front} $\overline {\mathcal C}_r \subseteq \mathbb R_+^2$ of the standard risk $E(w)$ and the adversarial risk proxy $\overline E(w,r)$, i.e.
\begin{eqnarray}
\overline {\mathcal C}_r = \{(E(w(r,t)),\overline E((w(r,t)))) \mid t \ge 0\},    
\end{eqnarray}
where $w(r,t)$ is the unique minimizer of $L_t(w,r) := tE(w) + \overline E(w,r)$ over $w \in \mathbb R^d$.
\begin{restatable}{prop}{front}
\label{prop:front}
For an attack strength $r \ge 0$ w.r.t to general norm $\|\cdot\|$, it holds that
\begin{align}
\overline{\mathcal C}_r &= \{(\sigma^2 + G(\lambda), \sigma^2 + F(r,\lambda)) \mid \lambda \in [0,r^2]\},\\
\overline{\mathcal C}_r &= \{(\sigma^2 + \epsilon^2 \|w_0\|_\Sigma^2, \overline E(w^{prox}(\lambda_{opt}(r,\epsilon),r)) \mid \epsilon \in [0,1]\}.
\end{align}
\end{restatable}
In \cite{Javanmard2020PreciseTI}, the Pareto front was computed for the case of Euclidean-norm attacks on isotropic features allowing them to obtain explicit tradeoffs in this setting.
\begin{proof}
Consider the bijective map $t \mapsto \lambda(r,t) := r^2/(1+t)$ from $[0,\infty]$ to $[0,r^2]$ and observe that $w(r,t) = w^{prox}(\lambda(r,t))$. The first part of the result then follows from the definition of $\overline{\mathcal C}_r$.

For the second part, observe from the definition of $\lambda_{opt}(r,\epsilon)$ in Lemma \ref{lm:implicit} that $G(\lambda_{opt}(r,\epsilon)) = \epsilon^2\|w_0\|_\Sigma^2$. The result then follows from the first part.
\end{proof}

\section{Spectral Analysis of Euclidean-Norm Attacks}
In the case of Euclidean-norm attacks, it turns out that the functions $F$ and $G$ are completely given in terms of spectral information as we now show. Let $\Sigma = \sum_{k \ge 1} \lambda_k \phi_k\phi_k^\top$ be the eigenvalue-decomposition of the feature covariance matrix $\Sigma$ and $c_k = \phi_k^\top w_0$ be the $k$th alignment coefficient of the generative model $w_0$, so that $w_0 = \sum_{k \ge 1} c_k \phi_k$. We shall occasionally consider the infinite-dimensional case where $d=\infty$, and we will require
\begin{eqnarray}
    \mbox{tr}(\Sigma) = \sum_k \lambda_k < \infty,
\end{eqnarray}
which ensures that the covariance operator $\Sigma$ is of \emph{trace class}.
Furthermore, it is easy to see that
\begin{eqnarray}
    \|w_0\|_2^2 = \sum_k c_k^2,\,\|w_0\|_\Sigma^2 = \sum_k \lambda_k c_k^2.
\end{eqnarray}
In this setting, the following result shows that the functions $F$ and $G$ \eqref{eq:FG} are given explicitly in terms of the spectral information $(\lambda_k,c_k^2)_{k \ge 1}$.
\begin{restatable}{lm}{}
For Euclidean-norm attacks, it holds for any $r,\lambda \ge 0$ that
\begin{eqnarray}
    G(\lambda) = \lambda^2\sum_k \frac{\lambda_k c_k^2}{(\lambda_k + \lambda)^2},\, F(r,\lambda) = G(\lambda) + r^2\sum_k  \frac{\lambda_k^2 c_k^2}{(\lambda_k + \lambda)^2}.
\end{eqnarray}
In particular, for $\lambda=r^2$, it holds that
\begin{eqnarray}
    F(r,r^2) = r^2\sum_k \frac{\lambda_k c_k^2}{\lambda_k + r^2}.
\end{eqnarray}
\label{lm:FGanalytic}
\end{restatable}
The proof (which will be provided shortly) relies on observing that $w^{prox}(\lambda) = (\Sigma + \lambda I)^{-1}\Sigma w_0$.
\subsection{Robustness and Statistical Dimension} Recall that,
for any $\lambda \ge 0$, the statistical dimension of $\Sigma$ is defined by $d_\Sigma(\lambda) := \mathrm{tr}(\Sigma(\Sigma + \lambda I)^{-1}) = \sum_k \lambda_k / (\lambda_k + \lambda)$. The following result highlights the role of the statistical dimension on adversarial robustness in the case of uniform source condition.
\begin{restatable}{cor}{}
If $c_k^2 \asymp c^2$ for some constant $c>0$, then $E_{opt}(r) \asymp \sigma^2 + c^2 r^2 d_\Sigma(r^2)$ for all $r \ge 0$.
\end{restatable}
\begin{proof}
Indeed, applying the second part of Lemma \ref{lm:FGanalytic} gives
$$
F(r,r^2) = r^2 \sum_k \lambda_k c_k^2 /(\lambda_k + r^2) = c^2 r^2 d_\Sigma(r^2).
$$
The result then follows directly from Theorem \ref{thm:Eopt}.
\end{proof}

\subsection{Proof of Lemma \ref{lm:FGanalytic}}
Indeed, in this setting, for any $\lambda \ge 0$, the solution of \eqref{eq:wt}  is explicitly given by $w^{prox}(\lambda) = (\Sigma + \lambda I)^{-1}\Sigma w_0$.
Plugging this into the definition of $F$ and $G$ given in \eqref{eq:FG} respectively then gives
\begin{align*}
G(\lambda) &= \|w^{prox}(\lambda)-w_0\|_\Sigma^2 = \|(\Sigma+\lambda I_d)\Sigma w_0-w_0\|_\Sigma^2\nonumber \\
&= \lambda^2\|(\Sigma + \lambda I)^{-1} w_0\|_\Sigma^2 = \lambda^2 \sum_k \frac{\lambda_k c_k^2}{(\lambda_k + \lambda)^2}
\end{align*}
and
\begin{align*}
F(r,\lambda) &= G(\lambda) + r^2\|w^{prox}(\lambda)\|_2^2 = G(\lambda) + r^2\|(\Sigma + \lambda I)^{-1}\Sigma w_0\|_2^2\nonumber \\
&= G(\lambda) + r^2\sum_k \frac{\lambda_k^2 c_k^2}{(\lambda_k+\lambda)^2},
\end{align*}
which proves the first part of the claim.

For the second part, applying the first part with $\lambda=r^2$ gives
\begin{eqnarray}
\begin{split}
F(r,r^2) &= \sum_k \frac{(r^4\lambda_k + r^2\lambda_k^2)c_k^2}{(\lambda_k + r^2)^2} = r^2\sum_k \frac{\lambda_k c_k^2}{\lambda_k + r^2}, 
\end{split}
\end{eqnarray}
where the second step is a basic algebraic manipulation.


\subsection{Proof of Theorem \ref{thm:isotropic}}
\isotropic*
\begin{proof}
Indeed, for any $\lambda \ge 0$, one easily computes
\begin{align}
G(\lambda) &= \lambda^2 \sum_{k=1}^d \frac{c_k^2}{(1 +\lambda)^2} = \dfrac{\lambda^2\|w_0\|_\Sigma^2}{(1+\lambda)^2},\\
F(r,\lambda) &= G(\lambda) + r^2\sum_{k=1}^d \frac{c_k^2}{(1+\lambda )^2} = \frac{(r^2+ \lambda^2)\|w_0\|_\Sigma^2}{(1+\lambda)^2}.
\end{align}
Now, one easily computes the free lunch threshold as $\epsilon_{FL}(r) = \sqrt{G(r^2)} / \|w_0\|_\Sigma= r^2/(1+r^2) \in [0,1]$. Observe that $\epsilon_{FL}(r) \asymp \min(r^2,1) \in [0,1]$. We then deduce from Theorem \ref{thm:freelunch} that the absolute optimal adversarial risk is given by
\begin{align*}
E_{opt}(r) &\asymp \sigma^2 + F(r,r^2) = \sigma^2 + \dfrac{r^2+r^4 }{(1+r^2)^2}\|w_0\|_2^2 = \sigma^2 + \dfrac{r^2}{1+r^2}\|w_0\|_2^2 \\
&\asymp \sigma^2 + \min(r^2,1)\|w_0\|_\Sigma^2.
\end{align*}
Moreover, if $\epsilon \ge \epsilon_{FL}(r)$, then $E_{opt}(r,\epsilon) \asymp E_{opt}(r)$ and there is free lunch: no tradeoff is required between accuracy and robustness. On the other hand, if $\epsilon \in [0,\epsilon_{FL}(r))$,
then solving the equation $G(\lambda) = \epsilon^2\|w_0\|_\Sigma^2$, we deduce that the optimal regularization parameter is given by
\begin{eqnarray}
\label{eq:isotropic-epsstar}
\lambda_{opt}(\epsilon) = \frac{ \epsilon}{1-\epsilon}.
\end{eqnarray}
Consequently, Theorem \ref{thm:freelunch} tells us that that
\begin{eqnarray*}
\begin{split}
    E_{opt}(r,\epsilon) \asymp \sigma^2 + F(r,\lambda_{opt}(\epsilon)) &= \sigma^2 + \frac{(r^2 + \epsilon^2/(1-\epsilon)^2))\|w_0\|_\Sigma^2}{(1+\epsilon/(1-\epsilon))^2}\\
    &= \sigma^2 + \|w_0\|_\Sigma^2(\epsilon^2  + (1-\epsilon)^2 r^2)\\
    &= \sigma^2 + \|w_0\|_\Sigma^2(\epsilon^2 + (1-\epsilon)^2 r^2),
    \end{split}
\end{eqnarray*}
from which the result follows.
\end{proof}
\subsection{Proof of Theorem \ref{thm:polydecay}}
\polydecay*
 We will need the following crucial lemma.

 \begin{restatable}{lm}{}
 Let the sequence $(\lambda_k)_{k \ge 1}$ of positive numbers be such that $\lambda_k \asymp k^{-\beta}$ for some constant $\beta > 0$, and let $m,n \ge 0$ with $n\beta > 1$. Then, for $D \gg 1$, it holds that 
 \begin{eqnarray}
     \sum_{k=1}^\infty \frac{\lambda_k^n}{(1+D\lambda_k)^m} \asymp D^{-c}\begin{cases}
         \log D,&\mbox{ if }m=n-1/\beta,\\
         1,&\mbox{ else,}
     \end{cases}
 \end{eqnarray}
 where $c:=\min(m,n-1/\beta) \ge 0$.
 \label{lm:fracture}
 \end{restatable}
 \begin{proof}
    First observe that
    \begin{eqnarray*}
        \begin{split}
    \lambda_k^n/(1+D\lambda_k)^m &\asymp \lambda_k^n\min(1,(D\lambda_k)^{-m})\\
    &= \begin{cases}
        \lambda_k^n
        =k^{-n\beta},&\mbox{ if }D\lambda_k < 1, \text{ i.e if }k > D^{1/\beta},\\
        D^{-m}\lambda_k^{-(m-n)}=D^{-m}k^{(m-n)\beta},&\mbox{ else.}
    \end{cases} 
        \end{split}
    \end{eqnarray*}
We deduce that
\begin{eqnarray}
\sum_{k=1}^\infty \frac{\lambda_k^n}{(1+D\lambda_k)^m} \asymp D^{-m}\sum_{1 \le k \le D^{1/\beta}}k^{(m-n)\beta} + \sum_{k > D^{1/\beta}}k^{-n\beta}.
\label{eq:fracture}
\end{eqnarray}
By comparing with the corresponding integral, one can write the first sum in \eqref{eq:fracture} as
\begin{eqnarray*}
    \begin{split}
D^{-m}\sum_{1 \le k \le D^{1/\beta}}k^{(m-n)\beta} &\asymp D^{-m}\int_1^{D^{1/\beta}}u^{(m-n)\beta}\mathrm{d}u\\
&\asymp D^{-m}
\begin{cases}
    (D^{1/\beta})^{1+(m-n)\beta}=D^{-(n-1/\beta)},&\mbox{ if }n - 1/\beta <  m,\\
    \log D,&\mbox{ if }m=n-1/\beta,\\
    1,&\mbox{ else.} 
\end{cases}\\
&=
\begin{cases}
    D^{-(n-1/\beta)},&\mbox{ if }n - 1/\beta <  m,\\
    D^{-m}\log D,&\mbox{ if }m=n-1/\beta,\\
    D^{-m},&\mbox{ else.}  
\end{cases}\\
&= D^{-c}
\begin{cases}
    \log D,&\mbox{ if }m=n-1/\beta,\\
    1,&\mbox{ else,}  
\end{cases}
    \end{split}
\end{eqnarray*}
where $c \ge 0$ is as given in the lemma.

Analogously, one can write the second sum in \eqref{eq:fracture} as
\begin{eqnarray*}
    \begin{split}
\sum_{k > D^{1/\beta}}k^{-n\beta} \asymp \int_{D^{1/\beta}}^\infty u^{-n\beta}\mathrm{d}u \asymp (D^{1/\beta})^{1-n\beta} = D^{-(n-1/\beta)},
    \end{split}
\end{eqnarray*}
and the result follows upon putting things together.
 \end{proof}

\begin{restatable}{cor}{}
\label{cor:polydecay}
Let $\beta$ and $\delta$ be as Theorem \ref{thm:polydecay}. For any $r \ge 0$ and small $\lambda > 0$, the functions $G$ and $F$ defined in \eqref{eq:FG} satisfy
\begin{align}
G(\lambda) &\asymp
\begin{cases}
\lambda^{1-\theta},&\mbox{ if }0 \le \delta < \beta+1,\\
\lambda^2\log(1/\lambda),&\mbox{ if }\delta=\beta+1,\\
\lambda^2,&\mbox{ if }\delta > \beta+1.
\end{cases}\label{eq:polyG}
\\
F(r,\lambda)
&\asymp \begin{cases}
\lambda^{1-\theta} + r^2\lambda^{-\theta},&\mbox{ if }0 \le \delta < 1,\\
\lambda^{1-\theta} + r^2\log(1/\lambda),&\mbox{ if }\delta=1,\\
\lambda^{1-\theta} + r^2,&\mbox{ if }1 < \delta < \beta + 1,\\
\lambda^2 \log(1/\lambda) + r^2,&\mbox{ if }\delta = \beta + 1,\\
\lambda^2 + r^2,&\mbox{ if }\delta > \beta+1.
\end{cases}\label{eq:polyF}.
\end{align}
Moreover, for small $r > 0$, it holds that
\begin{eqnarray}
F(r,r^2)
\asymp \begin{cases}
r^{2(1-\theta)},&\mbox{ if }0 \le \delta < 1,\\
r^2\log(1/r),&\mbox{ if }\delta=1,\\
r^2,&\mbox{ if }\delta > 1.
\end{cases}
\label{eq:polyFr2}
\end{eqnarray}
\end{restatable}
\begin{proof}
Set $D := 1/\lambda$. One can write
\begin{eqnarray}
    G(\lambda) = \lambda^2\sum_{k \ge 1}\frac{\lambda_k c_k^2}{(\lambda + \lambda_k)^2} = \sum_{k \ge 1}\frac{\lambda_k c_k^2}{(1+D\lambda_k)^2} \asymp \sum_{k \ge 1}\frac{\lambda_k^{1+\delta/\beta}}{(1+D\lambda_k)^2},
\end{eqnarray}
where we have used the fact that $\lambda_k c_k^2 \asymp k^{-\beta - \delta} = k^{-(1+\delta/\beta)\beta} \asymp \lambda_k^{1+\delta/\beta}$. Applying Lemma \ref{lm:fracture} with $n=1+\delta/\beta$ and $m=2$, we deduce that
\begin{eqnarray}
    G(\lambda) \asymp D^{-c}
    \begin{cases}
        \log D,&\mbox{ if }m=n-1/\beta,\text{ i.e if }\delta=\beta+1,\\
        1,&\mbox{ else,}
    \end{cases}
\end{eqnarray}
where $c = \min(m,n-1/\beta) = \min(2,1+\delta/\beta - 1/\beta) = \min(2,1-\theta)$. This proves \eqref{eq:polyG}.

Analogously, one can rewrite
\begin{eqnarray*}
    \begin{split}
        (F(r,\lambda) - G(\lambda))/r^2 = \sum_{k \ge 1}\frac{\lambda_k^2 c_k^2}{(\lambda + \lambda_k)^2} = D^2 \sum_{k \ge 1}\frac{\lambda_k^2 c_k^2}{(1+D\lambda_k)^2} \asymp D^2\sum_{k \ge 1}\frac{\lambda_k^{2+\delta/\beta}}{(1+D\lambda_k)^2}.
    \end{split}
\end{eqnarray*}
Applying Lemma \ref{lm:fracture} to the RHS with $n=2+\delta/\beta$ and $m=2$ then gives
\begin{eqnarray*}
(F(r,\lambda) - G(\lambda))/r^2 \asymp D^2\sum_{k \ge 1}\frac{\lambda_k^{2+\delta/\beta}}{(1+D\lambda_k)^2} \asymp D^{C}
\begin{cases}
    \log D,&\mbox{ if }m=n-1/\beta,\text{ i.e if }\delta=1,\\
    1,&\mbox{ else,}
\end{cases}
\end{eqnarray*}
where $C = 2-\min(m,n-1/\beta) = 2-\min(2,2-\theta) = -\min(0,-\theta) = \max(\theta,0)$. Combining with \eqref{eq:polyG} proves \eqref{eq:polyF}.

Finally, \eqref{eq:polyFr2} follows by pluggin $\lambda=r^2$ in \eqref{eq:polyF} and simplifying.
\end{proof}

We are now ready to prove Theorem \ref{thm:polydecay}
\begin{proof}[Proof of Theorem \ref{thm:polydecay}]
Equipped with Corollary \ref{cor:polydecay}, first observe that if $\delta > 1$, then $E_{opt}(r) \asymp \sigma^2 + F(r,r^2) \asymp \sigma^2 + r^2$ which matches $E(w_0,r) \asymp \sigma^2 + r^2\|w_0\|_2^2 \asymp \sigma^2 + r^2$, and so no tradeoff is needed: the ground-truth model $w_0$ achieves the optimal level of robustness $E_{opt}(r)$.

Now, if $\delta \in [0,\beta+1)$, we deduce that for $r=o(1)$, the free lunch threshold is given by
\begin{eqnarray}
\epsilon_{FL}(r) :=  \frac{\sqrt{G(r^2)}}{\|w_0\|_\Sigma} \asymp r^{(1-\theta)}=o(1).
\end{eqnarray}
For $\epsilon \in [0,\epsilon_{FL}(r))$, solving the equation $G(\lambda) = \epsilon^2\|w_0\|_\Sigma^2$ for $\lambda \in [0,r^2]$ gives
\begin{eqnarray}
\lambda_{opt}(r,\epsilon) \asymp \epsilon^{2/(1-\theta)}.
\label{eq:powerlaw-topt}
\end{eqnarray}

On the other hand, if $\delta=\beta+1$, then 
$G(\lambda) \asymp \lambda^2 \log(1/\lambda)$ for $\lambda=o(1)$, and so
\begin{eqnarray}
\epsilon_{FL}(r) := \frac{\sqrt{G(r^2)}}{\|w_0\|_\Sigma} \asymp r^2\log(1/r).
\end{eqnarray}
For $\epsilon \in [0,\epsilon_{FL}(r))$, solving \eqref{eq:implicit} for $\lambda \in [0,r^2]$ then gives
\begin{eqnarray}
\lambda_{opt}(r,\epsilon) \asymp e^{W(-\Theta(\epsilon^2))/2},    
\end{eqnarray}
where $W$ is an appropriate branch of the Lambert function.

Finally, if $\delta > \beta + 1$, then $G(\lambda) \asymp \lambda^2$ for $\lambda=o(1)$, and so 
\begin{eqnarray}
    \epsilon_{FL}(r) := \frac{\sqrt{G(r^2)}}{\|w_0\|_\Sigma} \asymp r^2.
\end{eqnarray}
For $\epsilon \in [0,\epsilon_{FL}(r))$, solving \eqref{eq:implicit} for $\lambda \in [0,r^2]$ then gives
\begin{eqnarray}
    \lambda_{opt}(r,\epsilon) \asymp r.
\end{eqnarray}

Combining with Theorem \ref{thm:freelunch} and putting things together gives the estimates stated in Table \ref{tab:poly}.
\end{proof}

\section{Other Proofs}
\subsection{Proof of Theorem \ref{thm:isotropic-eps-different-p}}
\isotropicEps*
\begin{proof}
    First observe that $\|w_0\|_\Sigma^2 = \|w_0\|_2^2 = s$ and $\|\Sigma w_0\|_p = \|w_0\|_p =s^{1/p}$.
We deduce that $r_0 := \|w_0\|_\Sigma/\|w_0\|_\star = s^{1/2}/\|w_0\|_q = s^{1/2-1/q}$. Likewise, $r_1 := \|\Sigma w_0\| / \|w_0\|_\Sigma = s^{1/p-1/2}=r_0$, since $1/p+1/q = 1$ by definition. We conclude that the problem is well-conditioned in the sense of Section xyz, and the claimed formulae for $E_{opt}(r)$ and $E_{opt}(r,\epsilon)$ follow from Theorem \ref{thm:well-conditioned}.

Furthermore, in the special case of Euclidean-norm attacks (i.e $p=2$), the claimed formula for the free lunch threshold $\epsilon_{FL}(r)$ and the optimal regularization parameter $\lambda_{opt}(r,\epsilon)$ follow from Theorem \ref{thm:isotropic}.

To conclude the proof, we know establish \eqref{eq:dust}. Indeed, if $\sigma^2=o(1)$, $p \in [1,\infty)$, and $1 \ll s \le d$ and we take $r \asymp 1 / s^{1/q}$ in the limit $d \to \infty$, then we see from Table \ref{tab:isotropic-eps} that $E_{opt}(r) \asymp (s/d)\min(r\sqrt d,1)^2 \asymp  \sigma^2 + (s/d)\min(s^{1/q}\sqrt d,1)^2 = \sigma^2 + (s/d) = o(1)$ since $\sigma^2=o(1)$ and $s/d =o(1)$.

Also from the same table, one reads $E_{opt}(r,\eps) \asymp (s/d)H(r/r_0(p),\epsilon)^2$ with $r_0(p) = s^{1/p-1/2}/\sqrt d$. Now, for any fixed $\epsilon \in [0,1)$, one computes
$$
H(r/r_0(p),\epsilon) \asymp H(s^{1/2-1/p-1/q} \sqrt d,\epsilon) = H(\sqrt{d/s},\epsilon) = (\epsilon + (1-\epsilon)\sqrt{d/s}) \asymp (1-\epsilon)\sqrt{d/s},
$$
and so $E_{opt}(r,\epsilon) \asymp (s/d)H(r/r_0(p),\epsilon)^2 \asymp (1-\epsilon)^2$ as claimed.
\end{proof}

\subsection{Proof of Theorem \ref{thm:peetre}}
\peetre*
\begin{proof}
For any $a \in \mathbb R^d$ and $t \ge 0$, define $\kappa_a(t) := \inf_{u \in \mathbb R^d}t \|u-a\|_2 + \|u\|_1$. Let $H_d := \sum_{k=1}^d 1/k$ be $d$th harmonic number. In the particular case where $a = (1/(k H_d))_{k \in [d]} \in \mathbb R^d$, it was shown in \cite{CanonneDistributionTesting2019} that: $\kappa_a(t) = \dfrac{2\log t + O(1)}{\log d}$ for $1 \ll t \le \sqrt d$. Noting that $w_0=H_d a$ and taking $t=1/r$, we deduce that
\begin{eqnarray}
    \begin{split}
K(w,r) &:= \inf_{w \in \mathbb R^d} \|w-w_0\|_2 + r\|w\|_1 = \inf_{w \in \mathbb R^d} H_d\|w/H_d-a\|_2 + rH_d\|w/H_d\|_1\\
&= rH_d \cdot \inf_{u \in \mathbb R^d} r^{-1}\|u-a\|_2 + \|u\|_1\text{ with change of variable }u = w/H_d\\
&= rH_d \cdot \kappa_a(1/r).
    \end{split}
\end{eqnarray}
Thus, for $1/\sqrt d \le r = o(1)$, taking $t=1/r$ gives
$$
K(w,r) = \dfrac{rH_d(\log(1/r) + O(1))}{\log d} = r (\log(1/r) + o(1)),
$$
from which the first part of the claim follows.

For the second part, taking $r=1/\log d$ gives 
$$
E_{opt}(r) \asymp \sigma^2 + r^2\log(1/r)^2 = \sigma^2 + (\log^2 d/\log d)^2 = \sigma^2+o(1) = o(1).
$$
On the other hand, it is clear that for any $r \ge 0$,
$$
E(w_0,r) = \sigma^2 + r^2\|w_0\|_1^2 = \sigma^2 + r^2(\sum_{j=1}^d 1/j)^2 \asymp \sigma^2 + (r\log d)^2,
$$
Thus, for $r = 1/\log d$ and $\sigma^2=o(1)$, then $E(w_0,r) \asymp \sigma^2 + 1 = \Theta(1)$.
\end{proof}

\end{document}